%% file: main.tex
\documentclass[12pt]{colt2025}

\title[Asymptotic Normality of Generalized Low-Rank Matrix Sensing]{Asymptotic Normality of Generalized Low-Rank Matrix Sensing \\ via Riemannian Geometry}
\usepackage{times}

\coltauthor{%
 \Name{Osbert Bastani} \Email{obastani@seas.upenn.edu} \\
 \addr University of Pennsylvania
}

\usepackage{amssymb,amsmath}
\newtheorem{assumption}[theorem]{Assumption}

\input{macros}

\begin{document}

\maketitle

\begin{abstract}
\input{abstract}
\end{abstract}

\input{body}

\bibliography{refs}

\appendix
\clearpage

\input{appendix}

\end{document}

%% file: macros.tex
\newcommand{\grad}{\text{grad}\,}
\newcommand{\hess}{\text{hess}\,}
\renewcommand{\ker}{\text{ker}\,}
\newcommand{\im}{\text{im}\,}
\newcommand{\lift}{\text{lift}}
\newcommand{\proj}{\text{proj}}
\newcommand{\inj}{\text{inj}}
\newcommand{\tr}{\text{tr}}

%% file: abstract.tex
We prove an asymptotic normality guarantee for generalized low-rank matrix sensing---i.e., matrix sensing under a general convex loss $\bar\ell(\langle X,M\rangle,y^*)$, where $M\in\mathbb{R}^{d\times d}$ is the unknown rank-$k$ matrix, $X$ is a measurement matrix, and $y^*$ is the corresponding measurement. Our analysis relies on tools from Riemannian geometry to handle degeneracy of the Hessian of the loss due to rotational symmetry in the parameter space. In particular, we parameterize the manifold of low-rank matrices by $\bar\theta\bar\theta^\top$, where $\bar\theta\in\mathbb{R}^{d\times k}$. Then, assuming the minimizer of the empirical loss $\bar\theta^0\in\mathbb{R}^{d\times k}$ is in a constant size ball around the true parameters $\bar\theta^*$, we prove $\sqrt{n}(\phi^0-\phi^*)\xrightarrow{D}\mathcal{N}(0,(H^*)^{-1})$ as $n\to\infty$, where $\phi^0$ and $\phi^*$ are representations of $\bar\theta^*$ and $\bar\theta^0$ in the horizontal space of the Riemannian quotient manifold $\mathbb{R}^{d\times k}/\text{O}(k)$, and $H^*$ is the Hessian of the true loss in the same representation.

%% file: body.tex
\section{Introduction}

Matrix completion has emerged as an effective tool for learning latent structure in a broad range of settings such as recommender systems~\citep{mackey2011divide} and word embeddings~\citep{pennington2014glove}. In addition, low-rank structure has served as a toy model of representation learning in deep neural networks~\citep{du2018power}. As a consequence, there has been a great deal of interest in understanding the theoretical properties of low-rank matrix completion~\citep{candes2008exact,keshavan2009matrix,negahban2011estimation,candes2012exact,hardt2014understanding,ge2016matrix,ge2017no}. Much of this work focuses on bounding finite sample estimation error or on providing guarantees on optimization (e.g., proving that various algorithms recover a global minimum of the loss function despite its non-convexity), or both. More recently, there has been interest in the asymptotic statistical properties of low-rank matrix completion~\citep{chen2019inference,farias2022uncertainty}. These properties are important for statistical inference since it enables the construction of confidence intervals around parameter estimates; for instance, matrix completion has seen recent use in econometric techniques for analyzing causal effects~\citep{athey2021matrix}.

In this paper, we establish an asymptotic normality result for the maximum likelihood estimator for generalized low-rank matrix sensing. We consider a general loss of the form 
\begin{align}
\label{eqn:intro}
\bar{L}(\bar\theta)=\frac{1}{n}\sum_{i=1}^n\bar\ell(\langle X_i,\bar\theta\bar\theta^\top\rangle,y_i^*),
\end{align}
where $X_i\in\mathbb{R}^{d\times d}$ is a measurement matrix, $y_i^*$ is the measurement, $\bar\theta\in\mathbb{R}^{d\times k}$ are the parameters, and $\bar\ell:\mathbb{R}\times\mathbb{R}\to\mathbb{R}$ is a given loss function, where $\bar\ell(z,y^*)$ is the loss for prediction $z$ when the true observation is $y^*$. As described below, we use the ``bar'' notation to differentiate these parameters and loss functions from their equivalents on a certain quotient manifold.

When $\bar\ell(z,y^*)=(z-y^*)^2$, we obtain the traditional matrix sensing loss, which up to constants equals the negative log-likelihood when the observations are assumed to be Gaussian---i.e., $y^*=\langle X_i,\bar\theta^*\bar\theta^{*\top}\rangle+\epsilon_i$, where $\epsilon\sim\mathcal{N}(0,\sigma^2)$ for some $\sigma\in\mathbb{R}_{>0}$. A more general loss can be useful in many settings---e.g., when the outcome $y_i^*$ is binary, we can take $\bar\ell$ to be the logistic loss or cross-entropy loss; for instance, this loss can be used to train word embeddings~\citep{pennington2014glove}.

Then, our goal is to recover a set of ``ground truth'' parameters $\bar\theta^*\in\operatorname*{\arg\min}_{\bar\theta}\bar{L}^*(\bar\theta)$, where $\bar{L}^*$ is the expected loss (where the $y_i^*$ are assumed to be random variables):
\begin{align*}
\bar{L}^*(\bar\theta)=\mathbb{E}\left[\frac{1}{n}\sum_{i=1}^n\bar{\ell}(\langle X_i,\bar\theta\bar\theta^{\top}\rangle,y_i^*)\biggm\vert\{X_i\}_{i=1}^n\right]
\end{align*}
For now, we assume the observation matrices $\{X_i\}_{i=1}^n$ are given (i.e., the fixed-design setting); however, for asymptotic normality, we will assume they are also random variables (i.e., the random-design setting), which allows us to take the limit $n\to\infty$ and apply the central limit theorem.

In general, the challenge facing learning with low-rank matrices is the symmetry to the rotation group. In particular, if $\bar\theta^*$ is a minimizer of (\ref{eqn:intro}), then so is $\bar\theta^*U$ for any orthogonal matrix $U\in\text{O}(k)$, since $(\bar\theta^*U)(\bar\theta^*U)^\top=\bar\theta^*UU^\top\bar\theta^{*\top}=\bar\theta^*\bar\theta^{*\top}$. This symmetry poses challenges for uncertainty quantification. To understand why, note that since the symmetry is continuous, this redundancy holds even locally in the neighborhood of any minimizer $\bar\theta^*$. In particular, an ``infinitesimal rotation'' of $\bar\theta^*$ is also a minimizer. Such a rotation can be modeled as a small perturbation of an actual rotation matrix the identity---in particular, given an orthogonal matrix $U\in\text{O}(k)$, we let $U\approx I+A$ for some $A\in\mathbb{R}^{d\times d}$. Then, by the identity $UU^\top=I$, we know that $A$ should satisfy
\begin{align*}
I=(I+A)(I+A)^\top\approx I+A+A^\top\Rightarrow A+A^\top\approx0,
\end{align*}
where we ignore terms of order $\|A\|_F^2$. Thus, $A$ is a skew-symmetric matrix. Formally, infinitesimal rotations are captured by the \emph{Lie algebra} $\mathfrak{o}(k)=\{A\in\mathbb{R}^{k\times k}\mid A+A^\top=0\}$ of the orthogonal group $\text{O}(k)$, which is the tangent space of $\text{O}(k)$ at the identity matrix~\citep{lee2003introduction}.

Then, the fact that $\bar{L}$ is invariant under ``infinitesimal rotation'' of $\bar\theta^*$ is captured by the fact that the Hessian $\nabla_{\bar\theta}^2\bar{L}^*(\bar\theta^*)$ near the true parameters $\bar\theta^*$ is degenerate on the subspace $\{\bar\theta^*Z\mid Z\in\mathfrak{o}(k)\}$. This is an issue is because asymptotic normality of the maximum likelihood estimator is typically
\begin{align*}
\sqrt{n}(\bar\theta^0-\bar\theta^*)\xrightarrow{D}\mathcal{N}(0,\nabla_{\bar\theta}^2\bar{L}^*(\bar\theta^*)^{-1}),
\end{align*}
where $\bar\theta^0$ is the maximum likelihood estimator of $\bar\theta^*$. In other words, the covariance matrix of the asymptotic distribution equals the inverse Fisher information, which is exactly $\nabla_{\bar\theta}^2\bar{L}^*(\bar\theta^*)^{-1}$ assuming $\bar{L}$ is the negative log-likelihood (and under the appropriate regularity conditions). The issue is that if $\nabla_{\bar\theta}^2\bar{L}^*(\bar\theta^*)$ is degenerative, then this inverse does not exist, which breaks the asymptotic normality guarantee. Intuitively, this happens because asymptotic normality establishes convergence to the true parameters; if an infinitesimal rotation $\bar\theta^*U\approx\bar\theta^*(I+A)$ of $\bar\theta^*$ also minimizes the loss, then the maximum likelihood estimator may converge to either $\bar\theta^*$ or $\bar\theta^*U$. To make this notion more precise, consider the Taylor approximation of $\bar{L}^*$ around $\bar\theta^*U$:
\begin{align*}
\bar{L}^*(\bar\theta^*)=\bar{L}^*(\bar\theta^*U)
&\approx\bar{L}^*(\bar\theta^*(I+A)) \\
&\approx\bar{L}^*(\bar\theta^*)+D_{\bar\theta}\bar{L}^*(\bar\theta^*)\{\bar\theta^*A\}+\frac{1}{2}D_{\bar\theta}^2\bar{L}^*(\bar\theta^*)\{\bar\theta^*A,\bar\theta^*A\} \\
&=\bar{L}^*(\bar\theta^*)+\frac{1}{2}D_{\bar\theta}^2\bar{L}^*(\bar\theta^*)\{\bar\theta^*A,\bar\theta^*A\},
\end{align*}
where $D_{\bar\theta}\bar{L}^*(\bar\theta^*)\{Z\}$ is the first derivative of $\bar{L}^*$ (a linear form) evaluated at $Z\in\mathbb{R}^{d\times k}$, where $D_{\bar\theta}^2\bar{L}^*(\bar\theta^*)\{Z,W\}$ is the second derivative of $\bar{L}^*$ (a bilinear form) evaluated at $Z,W\in\mathbb{R}^{d\times k}$, and where the last line follows since $\bar\theta^*$ minimizes $\bar{L}^*$. It follows that $D_{\bar\theta}^2\bar{L}^*(\bar\theta^*)\{\bar\theta^*A,\bar\theta^*A\}=0$, so
\begin{align*}
\nabla_{\bar\theta}^2\bar{L}^*(\bar\theta^*)\{\bar\theta^*A,\bar\theta^*A\}=0
\end{align*}
since the Hessian is the matrix representing the second derivative. As a consequence, the Hessian is not invertible, so we cannot establish asymptotic in the usual sense.

Nevertheless, we may hope to recover an asymptotic normality guarantee by simply ``ignoring'' the directions along which the Hessian is degenerate, at least if these directions do not affect optimality. In particular, since $\bar\theta^*$ and $\bar\theta^*U$ are both valid solutions, we argue that uncertainty along the direction of an infinitesimal rotation can safely be ignored when quantifying uncertainty.

To formalize this notion, we turn to concepts from Riemannian geometry. In particular, we consider convergence on the \emph{quotient manifold} $\Theta=\bar\Theta/\text{O}(k)$; this manifold is constructed from the equivalence classes of parameters $\{\bar\theta U\mid U\in\text{O}(k)\}$, which can be given a manifold structure in a natural way. Intuitively, near $\bar\theta^*$, the relationship between the original manifold and its quotient is to ``project out'' directions along infinitesimal rotations $A\in\mathfrak{o}(k)$; thus, the Hessian on the quotient manifold $\Theta$ can be non-degenerate even if the original Hessian is degenerate. Furthermore, the Hessian on the quotient manifold can be expressed as a bilinear form on a subspace of the original parameter space---specifically, on the orthogonal complement $H_{\bar\theta^*}$ of the subspace
\begin{align*}
V_{\bar\theta^*}=\{Z\in\mathbb{R}^{d\times k}\mid Z=\bar\theta^*A\},
\end{align*}
i.e., $H_{\bar\theta^*}=V_{\bar\theta^*}^\perp$. Thus, we can establish asymptotic normality of $\bar\theta^0$ on $H_{\bar\theta^*}$. Assuming $\bar\theta^0\approx\bar\theta^*$, then perturbing $\bar\theta^0$ along directions in $V_{\bar\theta^*}$ correspond to infinitesimal rotations of $\bar\theta^0$, which preserve optimality; thus, they can safely be ignored.
\begin{theorem}[informal]
\label{thm:informal}
Let $\bar\theta^0$ be a minimizer of the empirical loss with $n$ samples, and let $\bar\theta^*$ be the ground truth parameters that are ``best aligned'' with $\bar\theta^0$. Assume that $\bar\theta^0$ is in a local neighborhood of $\bar\theta^*$ (with radius independent of $n$). Let $\{e_i\}_{i=1}^{d'}$ be an orthonormal basis for $H_{\bar\theta^*}$, let $H^*\in\mathbb{R}^{d'\times d'}$ be the matrix representing $\mathbb{E}[\nabla_{\bar\theta}^2\bar{L}^*(\bar\theta^*)]$ in these coordinates, and let $\phi^*$ and $\phi^0$ be the vectors representing $\bar\theta^*$ and $\bar\theta^0$ in these coordinates, respectively. Then, we have
\begin{align*}
\sqrt{n}(\phi^0-\phi^*)\xrightarrow{D}\mathcal{N}(0,(H^*)^{-1})\qquad\text{as}\qquad n\to\infty.
\end{align*}
\end{theorem}
That is, we obtain asymptotic normality on the subspace where the Hessian is non-degenerate, which corresponds exactly to the directions along which we expect our parameter estimates to converge.

\paragraph{Related work.}

The most closely related work is \citet{chen2019inference}; Theorem 5 in their work provides finite-sample parameter estimation bounds for the low-rank factors in the low-rank matrix completion. The main difference is that our bounds allow for a nonlinear link function. Our work also focuses on low-rank matrix sensing instead of matrix completion, but we expect our results can be adapted to matrix completion. In addition, even with a linear link function, we obtain asymptotically tighter dependence on key constants; roughly speaking, translated to our setting, their sample complexity scales as $\|\bar\theta^0U-\bar\theta\|_{2,\infty}\le\tilde{O}(\sqrt{d^3k/n})$, whereas our corresponding bound would be $\|\bar\theta^0U-\bar\theta\|_{2,\infty}\le\tilde{O}(\sqrt{d^2k^2/n})$ (with an extra factor of $\sqrt{d}$ coming from the difference in norms). One shortcoming of our analysis is that it is restricted to symmetric matrices; the extension to non-symmetric matrices requires analyzing a more complex quotient manifold, which we leave to future work. In addition, \citet{farias2022uncertainty} provides an extension to more complex noise.

More broadly, there has been a great deal of interest in optimization on Riemannian manifolds, including over the space of low-rank matrices; see \citet{boumal2023introduction} for an introduction. This line of work has largely focused on optimization; there has been some work on extensions of Cram\'{e}r–Rao bounds to manifolds~\cite{smith2000intrinsic,boumal2013intrinsic,boumal2013robust,boumal2014cramer}, which provide lower bounds but do not establish asymptotic normality; furthermore, to the best of our knowledge, these approaches have not been applied to low-rank matrix sensing.

\section{Statement of Main Theorem}
\label{sec:lowrank}

\paragraph{Preliminaries.}

We formalize the generalized low-rank matrix sensing problem. Traditionally, we assume there is a ground truth matrix $M^*\in\mathbb{R}^{d\times d}$, and we obtain linear measurements of $M^*$:
\begin{align*}
y_i^*=\langle X_i,M^*\rangle+\epsilon_i
\qquad(\forall i\in[n]=\{1,...n\}),
\end{align*}
where $X_i\in\mathbb{R}^{d\times d}$ are \emph{measurement matrices}, $\epsilon_i\in\mathbb{R}$ are noise terms, $y_i^*\in\mathbb{R}$ are the measurements, and $\langle X,X'\rangle=\sum_{i,j=1}^dX_{ij}X_{ij}'$ is the Frobenius inner product. For instance, if $X_i$ is a one-hot matrix, then $y_i^*$ is a noisy measurement of a single entry of $M^*$.

Without any additional structure, this problem is a linear model, and we can estimate $M^*$ using linear regression. A common assuming is that $M^*$ has low-rank structure---i.e., $\text{rank}(M^*)=k$ for some $k\ll d$. We assume that $k$ is known (this assumption can be relaxed to unknown $k$ by trying different values of $k$ and using the best fit). In addition, we focus on the case where $M^*$ is symmetric; we discuss how this limitation might be relaxed in Section~\ref{sec:discussion}. Then, we can use the parameterization $M^*=\bar\theta^*\bar\theta^{*\top}$, where $\bar\theta^*\in\{\bar\theta\in\mathbb{R}^{d\times k}\mid\text{rank}(\bar\theta)=k\}$ is the subspace of rank $k$ matrices, and our goal is to recover the ground truth parameters $\bar\theta^*$.

When the measurements are linear and $\epsilon_i\sim\mathcal{N}(0,\sigma^2)$ are i.i.d. noise terms for some $\sigma\in\mathbb{R}_{>0}$, then up to additive and multiplicative constants, the negative log-likelihood is
\begin{align*}
\bar{L}^0(\bar\theta)=\frac{1}{n}\sum_{i=1}^n(\langle X_i,\bar\theta\bar\theta^\top\rangle)-y_i^*)^2.
\end{align*}
In many applications, however, the measurements may be nonlinear. For instance, when training word embeddings, $y_i^*$ might be a binary indicator of whether some relationship between words occurs. Then, a natural choice is to model $y_i^*$ as a Bernoulli random variable with mean $\sigma(\langle X_i,\bar\theta^*\bar\theta^{*\top}\rangle)$, where $\sigma(z)=(1+e^{-z})^{-1}$ is the sigmoid function. In general, we consider an arbitrary loss $\bar\ell:\mathbb{R}\times\mathbb{R}\to\mathbb{R}$ (subject to assumptions described below), in which case the empirical loss is
\begin{align*}
\bar{L}^0(\bar\theta)=\frac{1}{n}\sum_{i=1}^n\bar\ell(\langle X_i,\bar\theta\bar\theta^\top\rangle,y_i^*).
\end{align*}
Typically, $\bar\ell(z,y^*)$ is the negative log-likelihood for prediction $z$ if the ground truth measurement is $y^*$. For example, if $y^*$ is Bernoulli, then $\bar\ell$ would be the logistic loss.

\paragraph{Assumptions.} We assume that $(X_i,y_i^*)$ are i.i.d. samples from a probability distribution $P$; then, we additionally define the expected loss (conditioned on the measurement matrices $X_i$) to be
\begin{align*}
\bar{L}^*(\bar\theta)=\mathbb{E}[\bar{L}^0(\bar\theta)\mid\{X_i\}_{i=1}^n]=\frac{1}{n}\sum_{i=1}^n\mathbb{E}[\bar\ell(\langle X_i,\bar\theta\bar\theta^\top\rangle,y_i^*)\mid X_i].
\end{align*}
We introduce the random variables $\epsilon_i=\bar\ell'(\langle X_i,\bar\theta^*\bar\theta^{*\top}\rangle,y_i^*)$, $\epsilon_i'=\bar\ell''(\langle X_i,\bar\theta^*\bar\theta^{*\top}\rangle,y_i^*)$, and $\epsilon_i''=\bar\ell'''(\langle X_i,\bar\theta^*\bar\theta^{*\top}\rangle,y_i^*)$, where the derivatives of $\bar\ell$ are all with respect to its first argument. For intuition on our definition of $\epsilon$, consider the traditional matrix sensing problem, where $\bar\ell(z,y^*)=(z-y^*)^2$ and $y^*=\langle X,\bar\theta^*\bar\theta^{*\top}\rangle+\epsilon$, where $\epsilon\sim\mathcal{N}(0,\sigma^2)$. Then, $\bar\ell'(\langle X,\bar\theta^*\bar\theta^{*\top}\rangle,y^*)=\langle X,\bar\theta^*\bar\theta^{*\top}\rangle-y^*=\epsilon$. Our definition of $\epsilon$ generalizes this noise term to nonlinear losses. Finally, we denote the expected values of the noise terms by $\mu_i=\mathbb{E}[\epsilon_i\mid X_i]$, $\mu_i'=\mathbb{E}[\epsilon_i'\mid X_i]$, and $\mu_i''=\mathbb{E}[\epsilon_i''\mid X_i]$.
\begin{assumption}
\label{assump:lowrank:upperbound}
The following hold:
\begin{enumerate}
\item $\|\text{vec}(X_i)\|_{\infty}\le X_{\text{max}}$ (where $\text{vec}(X)\in\mathbb{R}^{dk}$ is the vectorization of $X$)
\item $\sigma_{\text{min}}\le\sigma_i\le\sigma_{\text{max}}$ for each $i\in[k]$, where $\sigma_1\ge...\ge\sigma_k$ are the $k$ singular values of $\bar\theta^*$
\item $\bar\ell'$, and $\bar\ell''$ are all $K_{\ell}$-Lipschitz continuous in their first argument
\item $\epsilon_i$, $\epsilon_i'$, and $\epsilon_i''$ are $\sigma_{\epsilon}$ subgaussian
\item $\mu_i=0$
\item $|\mu_i'|\le\mu_{\text{max}}$ and $|\mu_i''|\le\mu_{\text{max}}$ for all $i\in[n]$, and $\bar\ell$
\item $X_{\text{max}},\sigma_{\epsilon},\sigma_{\text{max}},\mu_{\text{max}},K_{\ell}\ge1$
\end{enumerate}
\end{assumption}
The first two assumptions are standard, bounding the measurement matrices and the ground truth parameters. The next four assumptions are regarding the general loss function $\bar\ell$. The main assumption of interest here is the fifth one, which says that $\mu_i=\mathbb{E}[\bar\ell'(\langle X_i,\bar\theta^*\bar\theta^{*\top}\rangle,y_i^*)\mid X_i]=0$. Intuitively, this assumption says that the ground truth parameters $\bar\theta^*$ optimizes not just the overall expected loss $\bar{L}^*$, but also the expected loss for the $i$th measurement---i.e., the model is well-specified. If $\bar\ell$ is the negative log-likelihood, this assumption is equivalent to the first Bartlett identity. The last assumption is for convenience, so we can drop these constants in our bounds when they are small.
\begin{assumption}
\label{assump:lowrank:lowerbound}
The following hold:
\begin{enumerate}
\item There exists $\mu_0\in\mathbb{R}_{>0}$ such that $\mu_i'\ge\mu_0$ for all $i\in[n]$
\item There exists $\lambda_0\in\mathbb{R}_{>0}$ such that for all $M\in\mathbb{R}^{d\times d}$ satisfying $\text{rank}(M)\le2r$, we have $\mathbb{E}[\langle X,M\rangle^2]\ge\lambda_0\|M\|_F^2$.
\item $\mu_0,\lambda_0\le1$.
\end{enumerate}
\end{assumption}
The first assumption says that the loss $\bar\ell$ is $\mu_0$-strongly convex in its first argument. The second generalizes a standard assumption on the covariance matrix (see, e.g., (10) in \citet{negahban2011estimation} or Definition 3 in \citet{ge2017no}). For intuition, suppose $\bar\ell(\langle X,\bar\theta^*\bar\theta^{*\top}\rangle,y^*)=(\langle X,\bar\theta^*\bar\theta^{*\top}\rangle-y^*)^2$; then, the first assumption holds since $\mu_i'=1$, and the second is closely related to the assumption that the minimum eigenvalue of the empirical covariance matrix has minimum eigenvalue at least $\lambda_0$. In particular, in linear regression, $X$ correspond to a covariate vector, in which case the covariance matrix is $\Sigma=\mathbb{E}[XX^\top]$, which corresponds to the bilinear form $(M,N)\mapsto M^\top\Sigma N=\mathbb{E}[\langle X,M\rangle\cdot\langle X,N\rangle]$.

If Assumption~\ref{assump:lowrank:lowerbound} were to hold for all $M$, then the minimum eigenvalue of this bilinear form is $\min_MM^\top\Sigma M/\|M\|^2\ge\lambda_0$. Since we are restricting to matrices $M$ with rank at most $2r$, we only ask it to hold for such $M$, so $\lambda_0$ is a kind of restricted minimum eigenvalue. The third assumption is again for convenience, so we can drop these constants in our bounds when they are small.
\begin{assumption}
\label{assump:lowrank:bartlett}
Let $\bar\ell^0(\bar\theta)=\bar\ell(\langle X,\bar\theta\bar\theta^\top\rangle,y^*)$, where $(X,y^*)\sim P$. Then, $\text{Cov}(\nabla_{\bar\theta}\bar\ell^0(\bar\theta))=\mathbb{E}[\nabla_{\bar\theta}^2\bar\ell^0(\bar\theta^*)]$.
\end{assumption}
In other words, the covariance of the gradient of a single sample of the loss equals the expected Hessian of the same sample. Note that the covariance and expectation are with respect to both $X$ and $Y^*$, not just $y^*$; in practice, we might expect it to hold conditional on $X$, but this is not necessary for our results. When $\bar\ell$ is the negative log-likelihood, this assumption is equivalent to the second Bartlett identity. It enables us to simplify the asymptotic distribution in our main result.

\paragraph{Formalizing the main theorem.}

Let $\bar\Theta\subseteq\mathbb{R}^{d\times k}$ denote the manifold of valid solutions (including both our rank constraint as well as assumptions described below), and consider the estimator $\bar\theta^0\in\operatorname*{\arg\min}_{\bar\theta\in\bar\Theta}\bar{L}^0(\bar\theta)$. That is, $\bar\theta^0$ is a minimizer of the empirical loss given observations $\{(X_i,y_i^*)\}_{i=1}^n$; recall that it is not unique. Our goal is to analyze the convergence of $\bar\theta^0$ to the ``best aligned'' ground truth parameters. In particular, letting $\bar\Theta^*=\operatorname*{\arg\min}_{\bar\theta\in\bar\Theta}$ $\bar{L}^*(\bar\theta)$ be the ground truth parameters, then the best aligned ground truth parameters are $\bar\theta^*=\operatorname*{\arg\min}_{\bar\theta\in\bar\Theta^*}\|\bar\theta^0-\bar\theta^*\|_F$, where $\|X-Y\|_F=\langle X-Y,X-Y\rangle$ is the Frobenius norm. Now, our goal is to establish asymptotic convergence of the minimizer $\bar\theta^0$ of the empirical loss $\bar{L}^0$ to the ground truth parameters $\bar\theta^*$, but that due to the symmetry in the loss function, we can only establish this guarantee on a subspace
\begin{align*}
H_{\bar\theta^*}=\{\bar\theta^*A\in\mathbb{R}^{d\times k}\mid A\in\mathfrak{o}(k)\}^\top.
\end{align*}
One key challenge in formalizing Theorem~\ref{thm:informal} is choosing a basis $\{e_i\}$ of this subspace. The issue is that the subspace depends on $\bar\theta^*$, and in the informal theorem statement, $\bar\theta^*$ is chosen to be the ground truth that are best aligned with $\bar\theta^0$, making $\bar\theta^*$ itself a random variable. Further complicating the issue, this randomness includes the randomness from the process by which $\bar\theta^0$ is chosen among the set of minimizers of $\bar{L}^0$, but we do not want to make any assumptions about this process.

To address this issue, we need to choose the basis in a way that the representations of the relevant vectors and matrices in this basis (namely, $\phi^*$, $\phi^0$, and $H^*$) do not depend on the choice of $\bar\theta^0$ among optimizers of the empirical loss. To this end, we use the following strategy. Rather than let $\bar\theta^*$ vary with $\bar\theta^0$, we fix some arbitrary choice of ground truth parameters $\bar\theta^*$. Then, we choose an orthonormal basis $\{e_i\}_{i=1}^{d'}$ for $H_{\bar\theta^*}$; for a rotation $\bar\theta^{*\prime}=\bar\theta^*U$ of $\bar\theta^*$, we choose the basis $\{e_iU\}_{i=1}^n$.\footnote{Mathematically, $H_{\bar\theta^*}$ is a subspace of the tangent space of $\mathbb{R}^{d\times k}$ at $\bar\theta^*$, and $e_iU$ is the pushforward of $e_i$ under the group action $g_U(\bar\theta)=\bar\theta U$ (noting that the derivative of $g_U$ is the linear form $D_{\bar\theta}g_U(\bar\theta)[e]=eU$ since $g_U$ is linear). This ensures that the gradient and Hessian are invariant, since the function being invariant under the group action implies that the gradient and Hessian are invariant under the pushforward of the group action.} Since $U$ is orthonormal, this transformed basis is also an orthonormal basis.

To see formalize the invariance property, let $\bar\theta^0$ denote the minimizer of the empirical loss that is best aligned. Now, suppose that we obtain $\bar\theta^{0\prime}=\bar\theta^0U$ (where $U\in\text{O}(k)$). It is easy to check that $\bar\theta^{0\prime}$ is best aligned with ground truth parameters $\bar\theta^{*\prime}=\bar\theta^*U$. Now, let $\bar\ell^0(\bar\theta)=\bar\ell(\langle X,\bar\theta\bar\theta^\top\rangle,y^*)$, where $(X,y^*)\sim P$ is a single random observation, and let $\bar\ell^*(\bar\theta)=\mathbb{E}[\bar\ell^0(\bar\theta)]$ be its expectation. Now, the representations $\phi^0,\phi^*,g^0\in\mathbb{R}^{d'}$ of $\bar\theta^0,\bar\theta^*,\nabla_{\bar\theta}\bar\ell^0(\bar\theta^*)$, respectively, and $h^0,h^*\in\mathbb{R}^{d'\times d'}$ of $\nabla_{\bar\theta}^2\bar\ell^0(\bar\theta^*)$ and $\nabla_{\bar\theta}^2\bar\ell^*(\bar\theta^*)$ (where $d'=\text{dim}(H_{\bar\theta^*})$) induced if we choose the minimizer $\bar\theta^0$ are
$\phi^0_i=\langle\bar\theta^0,e_i\rangle$, $\phi^*_i=\langle\bar\theta^*,e_i\rangle$, $g_i^0=\langle\nabla_{\bar\theta}\bar\ell^0(\bar\theta^*),e_i\rangle$, $h_{ij}^0=\langle e_i,\nabla_{\bar\theta}^2\bar\ell^0(\bar\theta^*)e_j\rangle$, and $h_{ij}^*=\langle e_i,\nabla_{\bar\theta}^2\bar\ell^*(\bar\theta^*)e_j\rangle$. Here, we have included additional vectors/matrices that are needed in our analysis. If we instead choose the minimizer $\bar\theta^{0\prime}=\bar\theta^0U$, the representations $\phi^{0\prime},\phi^{*\prime},G^{0\prime}\in\mathbb{R}^{d'}$ and $h^{0\prime},h^{*\prime}\in\mathbb{R}^{d'\times d'}$ are $\phi^{0\prime}_i=\langle\bar\theta^{0\prime},e_i'\rangle$, $\phi^{*\prime}_i=\langle\bar\theta^{*\prime},e_i'\rangle$, $g_i^{0\prime}=\langle\nabla_{\bar\theta}\bar\ell^0(\bar\theta^{*\prime}),e_i'\rangle$, $h_{ij}^{0\prime}=\langle e_i',\nabla_{\bar\theta}^2\bar\ell^0(\bar\theta^{*\prime})e_j'\rangle$, $h_{ij}^{*\prime}=\langle e_i',\nabla_{\bar\theta}^2\bar\ell^*(\bar\theta^{*\prime})e_j'\rangle$.
\begin{lemma}
\label{lem:invariant}
We have $\phi^0=\phi^{0\prime}$, $\phi^*=\phi^{*\prime}$, $g^0=g^{0\prime}$, $h^0=h^{0\prime}$, and $h^*=h^{*\prime}$.
\end{lemma}
The proof is straightforward; we provide it in Appendix~\ref{sec:lem:invariant:proof}.

A second challenge formalizing Theorem~\ref{thm:informal} is that we need to formalize the notion of ``local neighborhood''. Since there are many possible ground truth parameters $\bar\Theta^*$, this neighborhood needs to be defined with respect to the entire set $\bar\Theta^*$ rather than a specific $\bar\theta^*\in\bar\Theta^*$. In particular, given a radius parameter $r\in\mathbb{R}_{>0}$, we consider the open set $\bar\Theta^0\subseteq\bar\Theta$ defined by
\begin{align*}
\bar\Theta^0(r)=\left\{\bar\theta\in\bar\Theta\biggm\vert\operatorname*{\arg\min}_{\bar\theta^*\in\bar\Theta^*}\|\bar\theta-\bar\theta^*\|_F\le r\right\}.
\end{align*}
Then, we focus on minimizers of $\bar{L}^0$ in $\bar\Theta^0(r)$, where $r$ is to be specified. The reason we need to assume a local neighborhood is that our analysis relies on local strong convexity on the quotient manifold, which only holds in a neighborhood of the projection of $\bar\theta^*$ onto that manifold. Importantly, $r$ does not depend on $n$ (but it does depend on other problem parameters such as $d$ and $k$), ensuring that it still makes sense for us to take the limit $n\to\infty$. Now, we state our main result.
\begin{theorem}
\label{thm:main}
Let $r\in\mathbb{R}_{>0}$ be a constant that does not depend on $n$, let $\bar\theta^0U\in\bar\Theta^0(r)$ be a minimizer of $\bar{L}^0$ with $n$ samples, and let $\bar\theta^*U$ be the ground truth parameters that are ``best aligned'' with $\bar\theta^0U$. Let $\{e_i\}_{i=1}^{d'}$ be an orthonormal basis for $H_{\bar\theta^*}$, so $\{e_iU\}_{i=1}^{d'}$ is an orthonormal basis for $H_{\bar\theta^*U}$. Let $\phi^*$ and $\phi^0$ be the vectors representing $\bar\theta^*U$ and $\bar\theta^0U$ in these coordinates, respectively, and let $H^*\in\mathbb{R}^{d'\times d'}$ be the matrix representing $\mathbb{E}[\nabla_{\bar\theta}^2\bar{L}^*(\bar\theta^*)]$. Then, under Assumptions~\ref{assump:lowrank:upperbound}, \ref{assump:lowrank:lowerbound}, \&~\ref{assump:lowrank:bartlett}, we have
\begin{align*}
\sqrt{n}(\phi^0-\phi^*)\xrightarrow{D}\mathcal{N}(0,(H^*)^{-1})\qquad\text{as}\qquad n\to\infty.
\end{align*}
\end{theorem}
We provide a proof in the Section~\ref{sec:thm:main:proof}.

\section{Proof of Main Theorem}

\subsection{Overview}
\label{sec:thm:main:proof:overview}

Traditionally, establishing asymptotic normality of maximum likelihood estimation proceeds by Taylor expanding the gradient of the empirical loss around $\bar\theta^*$:
\begin{align}
\label{eqn:taylor}
\nabla_{\bar\theta}\bar{L}^0(\bar\theta^0)\approx\nabla_{\bar\theta}\bar{L}^0(\bar\theta^*)+\nabla_{\bar\theta}^2\bar{L}^0(\bar\theta^*)(\bar\theta^0-\bar\theta^*).
\end{align}
The left-hand side is zero since $\bar\theta^0$ is an extremal point, so we can rearrange to obtain
\begin{align*}
\bar\theta^0-\bar\theta^*\approx-\nabla_{\bar\theta}^2\bar{L}^0(\bar\theta^*)^{-1}\nabla_{\bar\theta}\bar{L}^0(\bar\theta^*).
\end{align*}
Assuming $\bar\theta^0\xrightarrow{P}\bar\theta^*$, then by the weak law of large numbers and the continuous mapping theorem, we have $\nabla_{\bar\theta}^2\bar{L}^0(\bar\theta^*)^{-1}\xrightarrow{P}\nabla_{\bar\theta}^2\bar{L}^*(\bar\theta^*)^{-1}$; furthermore, by the central limit theorem, we have
\begin{align*}
\sqrt{n}\nabla_{\bar\theta}\bar{L}^0(\bar\theta^*)\xrightarrow{D}\mathcal{N}(0,\text{Cov}(\nabla_{\bar\theta}\bar{\ell}^0(\bar\theta^*))),
\end{align*}
where $\bar\ell^0$ is the loss for a single observation (so $\bar{L}^0(\bar\theta)=n^{-1}\sum_{i=1}^n\bar\ell^0(\bar\theta)$). Then, applying Slutsky's theorem and the delta method, and using the fact that $\text{Cov}(\nabla_{\bar\theta}\bar{\ell}^0(\bar\theta^*))=\nabla_{\bar\theta}^2\bar{L}^0(\bar\theta^*)$, we have
\begin{align*}
\sqrt{n}(\bar\theta^0-\bar\theta^*)\xrightarrow{D}\mathcal{N}(0,\nabla_{\bar\theta}^2\bar{L}^0(\bar\theta^*)^{-1}).
\end{align*}
Recall that in the low-rank matrix sensing setting, $\nabla_{\bar\theta}^2\bar{L}^0(\bar\theta^*)$ has zero eigenvalues corresponding to directions along the Lie algebra $\mathfrak{o}(k)$ (denoted $V_{\bar\theta^*}$), so it is not invertible, so we instead want to perform our analysis on the orthogonal complement of this subspace (denoted $H_{\bar\theta^*}=V_{\bar\theta^*}^\top$). In particular, we can restrict (\ref{eqn:taylor}) to this subspace and carry out the same analysis.

However, there is a remaining difficulty---in (\ref{eqn:taylor}), we are assuming that the remainder term in the Taylor expansion becomes small as $n\to\infty$. This holds if the estimator is consistent (i.e., $\bar\theta^0\xrightarrow{P}\bar\theta^0$); this fact is also required later in the proof to show that $\nabla_{\bar\theta}^2\bar{L}^0(\bar\theta^*)^{-1}\xrightarrow{L}\nabla_{\bar\theta}^2\bar{L}^*(\bar\theta^*)^{-1}$. However, our estimator is not consistent since it is not even identifiable. To address this issue, we instead analyze convergence on the \emph{quotient manifold} $\Theta=\bar\theta/\text{O}(k)$ of $\bar\Theta$ under the action of the group $\text{O}(k)$; intuitively, $\Theta$ is constructed from equivalence classes of $\bar\Theta$ under rotations $U\in\text{O}(k)$. If we can show that $\bar\theta^*$ is unique up to rotations, then we can show consistency on this manifold. We show a stronger result---namely, that $\bar\theta^0$ converges to $\bar\theta^*$ at a rate of $O(n^{-1/2})$ on this manifold. In particular, let $\theta^0$ and $\theta^*$ denote the elements of $\Theta$ corresponding to $\bar\theta^0$ and $\bar\theta^*$, respectively, and let $d:\Theta\times\Theta\to\mathbb{R}_{\ge0}$ be the distance function on $\Theta$. Then, we have the following result.
\begin{theorem}
\label{thm:lowrank}
Given $\delta\in\mathbb{R}_{>0}$, suppose that
\begin{align*}
n\ge\frac{640d^2k^6X_{\text{max}}^4\sigma_{\text{max}}^4\sigma_{\epsilon}^2\log(12d^2k^2/\delta)}{\mu_0\lambda_0\sigma_{\text{min}}^2},
\end{align*}
and that
\begin{align*}
d(\theta^0,\theta^*)\le\min\left\{\sigma_{\text{min}},\frac{\mu_0\lambda_0\sigma_{\text{min}}^3}{320X_{\text{max}}^4\sigma_{\text{max}}^5d^4k^{5/2}(K_{\ell}+\mu_{\text{max}}+15\sigma_{\epsilon})}\right\}.
\end{align*}
Then, under Assumptions~\ref{assump:lowrank:upperbound} \&~\ref{assump:lowrank:lowerbound}, with probability at least $1-\delta$, we have
\begin{align*}
d(\theta^0,\theta^*)\le\sqrt{\frac{512dk^2\sigma_{\text{max}}^2\sigma_{\epsilon}^2X_{\text{max}}^2\log(8dk/\delta)}{n\mu_0^2\lambda_0^2\sigma_{\text{min}}^4}}=\tilde{O}\left(\sqrt{\frac{dk^2}{n}}\right).
\end{align*}
\end{theorem}
We give a proof in Appendix~\ref{sec:thm:lowrank:proof}. At a high level, let $L^0:\Theta\to\Theta$ and $L^*:\Theta\to\Theta$ be the variants of $\bar{L}^0$ and $\bar{L}^*$ induced on $\Theta$, respectively. Then, our proof proceeds by establishing local strong convexity of $L^*$ (i.e., strongly convex in a neighborhood of $\theta^*$ whose radius is independent of $n$), which follows by showing that the minimum eigenvalue of the Hessian of $L^*$ is lower bounded. Since $L^0$ converges to $L^*$, this implies that its minimizer $\theta^0$ in a neighborhood of $\theta^*$ converges to $\theta^*$. This result enables us to adapt the standard proof of asymptotic normality to our setting.

In the remainder of this section, we establish the existence of the Riemannian quotient manifold $\Theta=\bar\Theta/\text{O}(k)$ (Section~\ref{sec:lowrankquotient}), provide a variant of Taylor's theorem for Riemannian manifolds (Section~\ref{sec:taylor}), sketch the derivation of $H_{\bar\theta^*}$ (Section~\ref{sec:horizontalvertical}, sketch the proof of Theorem~\ref{thm:lowrank} (Section~\ref{sec:thm:lowrank:proofsketch}, a complete proof is in Appendix~\ref{sec:thm:lowrank:proof}), and finally prove Theorem~\ref{thm:main} (Section~\ref{sec:thm:main:proof}).

\subsection{The Riemannian Quotient Manifold}
\label{sec:lowrankquotient}

Our proof proceeds by analyzing convergence on the \emph{quotient manifold} $\Theta=\bar\Theta/\text{O}(k)$. In this section, we establish the existence of $\Theta$ and describe its structure; our results are standard (e.g., see \cite{massart2020quotient,boumal2023introduction}), but we include them here for completeness. The points in the quotient manifold are equivalence classes of parameters under the rotational symmetry---i.e., a point $\theta\in\Theta$ has the form $\theta=\{\bar\theta U\mid U\in\text{O}(k),\bar\theta\in\bar\Theta\}$. Under certain conditions, this set of points can be given the structure of a Riemannian manifold\footnote{A \emph{Riemannian manifold} $\Theta$ is a smooth manifold where each point $\theta\in\Theta$ is associated with an inner product $\langle\cdot,\cdot\rangle_{\theta}$ mapping tangent vectors $Z,W\in T_{\theta}\Theta$ to $\langle Z,W\rangle_{\theta}\in\mathbb{R}_{\ge0}$; the corresponding norm is denoted $\|\cdot\|_{\theta}$. This association is smooth in the sense that for any smooth vector fields $V,W$ on $\Theta$, the function $\theta\mapsto\langle V(\theta),W(\theta)\rangle_{\theta}$ is smooth.} based on the structure of the original manifold $\bar\Theta$.
Specifically, denote the group action by $\alpha:\text{O}(k)\times\bar\Theta\to\bar\Theta$, where $\alpha(U,\bar\theta)=\bar\theta U$; then, the conditions for $\Theta$ to exist are established in the following lemma.
\begin{lemma}
The action $\alpha$ is free (i.e., if $\alpha(U,\bar\theta)=\bar\theta$, then $U=I$ is the identity), proper (i.e., preimages of compact sets are compact), and isometric (i.e., $\langle\bar\theta,\bar\theta'\rangle=\langle\alpha(\bar\theta,U),\alpha(\bar\theta',V\rangle$).
\end{lemma}
\begin{proof}
First, $\alpha$ is free since if $\bar\theta U=\bar\theta$, then letting $\bar\theta^\dagger$ be the pseudoinverse of $\bar\theta$, then we have $U=\bar\theta^\dagger\bar\theta U=\bar\theta^\dagger\bar\theta=I$. Second, $\alpha$ is proper since $\text{O}(k)$ is a compact manifold since it is closed and bounded, and group actions for compact groups are always proper. Finally, $\alpha$ is isometric since
\begin{align*}
\langle\bar\theta U,\bar\theta'U\rangle=\tr(U^\top\bar\theta^\top\bar\theta'U)=\tr(\bar\theta^\top\bar\theta'UU^\top)=\tr(\bar\theta^\top\bar\theta')=\langle\bar\theta,\bar\theta'\rangle.
\end{align*}
The claim follows.
\end{proof}
Thus, by Theorem 9.38 in \citet{boumal2023introduction}, $\Theta=\bar\Theta/\text{O}(k)$ is a Riemannian manifold. Now, the \emph{quotient map} $\pi:\bar\Theta\to\Theta$ is defined to be the function mapping $\bar\theta$ to its equivalence class $\theta$; we let $\theta^*=\pi(\bar\theta^*)$ and $\theta^0=\pi(\bar\theta^0)$. Furthermore, since $\bar{L}^*$ and $\bar{L}^0$ are invariant under the action of $\text{O}(k)$, we obtain induced maps $L^*:\Theta\to\mathbb{R}$ and $L^0:\Theta\to\mathbb{R}$. It is easy to check that $\theta^*\in\operatorname*{\arg\min}_{\theta\in\Theta}L^*(\theta)$ and $\theta^0\in\operatorname*{\arg\min}_{\theta\in\Theta}L^0(\theta)$, since the corresponding relations hold on $\bar\Theta$.

\subsection{A Taylor Expansion on Riemannian Manifolds}
\label{sec:taylor}

Establishing both asymptotic normality as well as Theorem~\ref{thm:lowrank} relies on an applying a Taylor expansion for Riemannian manifolds to the gradient of $L^0$. To describe this Taylor expansion, we need to define the Riemannian gradient and Hessian of $L^0$, as well as those of $L^*$; in this section, we begin by giving a brief background; for formal definitions, see \cite{boumal2023introduction}. 

In Euclidean space, given $f:\mathbb{R}^d\to\mathbb{R}^{d'}$ its derivative $Df(z):\mathbb{R}^d\to\mathbb{R}^{d'}$ at $z\in\mathbb{R}^d$ is a linear approximation of $f$ near $z$, i.e., $f(z+\epsilon)\approx f(z)+Df(z)[\epsilon]$. Here, $Df(z)[\epsilon]$ denotes multiplying $\epsilon\in\mathbb{R}^d$ by the matrix representing $Df(z)$; our notation emphasizes that $Df(z)$ is a linear map, not a matrix. Then, the gradient $\nabla_zf(z)$ is the unique vector such that $Df(z)[\epsilon]=\langle\nabla_zf(z),\epsilon\rangle$.

In general, given a mapping $f:\mathcal{M}\to\mathcal{N}$ between two manifolds, we might want to define its derivative to be a function $Df(z):\mathcal{M}\to\mathcal{N}$. However, this does not make sense since $\mathcal{M}$ and $\mathcal{N}$ may not be vector spaces. Instead, we formally associate each point $z\in\mathcal{M}$ with a \emph{tangent space} $T_z\mathcal{M}=\mathbb{R}^d$, where $d$ is the dimension of $\mathcal{M}$, and similarly for $\mathcal{N}$. Then, we define $Df(z)$ to be a linear map on tangent spaces $Df(z):T_z\mathcal{M}\to T_{f(z)}\mathcal{N}$. While we cannot directly use $Df(z)$ as a linear approximation, we can do so if we compose $f$ with maps to and from Euclidean space---i.e., if we have a composition $g=\psi\circ f\circ\phi$, where $\phi:\mathbb{R}^{d''}\to\mathcal{M}$ and $\psi:\mathcal{N}\to\mathbb{R}^{d'''}$, then we have
\begin{align*}
g(x+\epsilon)\approx g(x)+D\psi(f(\phi(x)))[Df(\phi(x))[D\phi(x)[\epsilon]]].
\end{align*}
Now, the gradient of $f$ at $z\in\mathcal{M}$ is the unique vector $\grad f(z)\in T_z\mathcal{M}$ such that $Df(z)[\epsilon]=\langle\grad f(z),\epsilon\rangle_z$. Next, defining the Hessian relies on the \emph{Levi-Civita connection} (denoted $\nabla$), a generalization of the directional derivative to manifolds. Roughly speaking, the Hessian $\hess f(z)[\epsilon]$ of $f$ at $z$ is the directional derivative of $\grad f(z)$ along $\epsilon$, i.e., $\hess f(z)[\epsilon]=\nabla_{\epsilon}\grad f(z)$.

One additional ingredient needed for our Taylor expansion is the \emph{exponential map}, which is a natural diffeomorphism $\exp_{\theta}:T_{\theta}\Theta\to\Theta$ with $\exp_{\theta}(0)=\theta$. In Euclidean space, this mapping is the identity; for general manifolds, $\exp(v)$ is the endpoint of the geodesic curve $\gamma:[0,1]\to\Theta$ satisfying $\gamma(0)=\theta$ and $\gamma'(0)=v$, whose existence and uniqueness can be established by results from differential equations theory. A key property of $\exp$ is that it is locally a diffeomorphism, and the \emph{injectivity radius} $\inj(\theta)\in\mathbb{R}_{\ge0}$ is the largest $r$ such that $\exp_{\theta}$ is invertible on $B(0,r)\subseteq T_{\theta}\Theta$ (see Corollary 10.25 in \cite{boumal2023introduction}). Finally, $\exp_{\theta}$ is a \emph{radial isometry}---i.e., for $\theta'\in B(0,\inj(\theta))$, letting $v=\exp_{\theta}^{-1}(\theta')$, we have $d(\theta,\theta')=\|v\|_{\theta}$ (see Proposition 10.22 in \cite{boumal2023introduction}).

With these facts in hand, we have the following Taylor expansion of $\grad L^0$ at $\theta^*$.
\begin{lemma}
\label{lem:taylor}
Consider a Riemannian manifold $\Theta$, smooth losses $L^0,L^*:\Theta\to\mathbb{R}_{\ge0}$, and minimizers $\theta^0\in\operatorname*{\arg\min}_{\theta\in\Theta}L^0(\theta)$ and $\theta^*\in\operatorname*{\arg\min}_{\theta\in\Theta}L^*(\theta)$. Assume that (i) $d(\theta^0,\theta^*)<\inj(\theta^*)$, and (ii) $\hess L^0$ is $K$-Lipschitz continuous for some $K\in\mathbb{R}_{>0}$.\footnote{Since $L^0$ is smooth, the $K$-Lipschitz property can be formulated as follows: $\|\nabla_w\hess L^0(\theta)\|_{\theta}\le K\|w\|_{\theta}$ for all $\theta\in\Theta$ and $w\in T_{\theta}\Theta$. See Definition 10.49 and Corollary 10.52 in \cite{boumal2023introduction} for details.}
Then, we have
\begin{align*}
\|\grad L^0(\theta^*)+\hess L^0(\theta^*)[v]\|_{\theta^*}\le\frac{K}{2}d(\theta^0,\theta^*)^2.
\end{align*}
\end{lemma}
\begin{proof}
By assumption (i), we can define $v=\exp_{\theta^*}^{-1}(\theta^0)$. With assumption (ii), Corollary 10.56 in \citet{boumal2023introduction} provides a Taylor expansion of the form
\begin{align*}
\|P_v^{-1}\grad L^0(\theta^0)-\grad L^0(\theta^*)-\hess L^0(\theta^*)[v]\|_{\theta^*}\le\frac{K}{2}\|v\|_{\theta^*}^2,
\end{align*}
where the \emph{parallel transport} $P_v^{-1}$ is used to ``transport'' $\grad L^0(\theta^0)$ (an element of $T_{\theta^0}\Theta$) to $T_{\theta^*}\Theta$. Briefly, this result follows by choosing a basis for $T_{\theta^*}\Theta$, in which case $\grad L^0$ can be expressed as a vector in $\mathbb{R}^d$ (where $d$ is the dimension of $\Theta$), enabling us to apply Taylor's theorem for real-valued functions. One subtlety the need to transport vectors from $T_{\theta}\Theta$ to $T_{\theta^*}\Theta$ for $\theta$ along the geodesic from $\theta^*$ to $\theta^0$, which is accomplished using parallel transport. Now, by the first-order condition on Riemannian manifolds (see Propositions 4.5 \& 4.6 in \citet{boumal2023introduction}), $\grad L^0(\theta^0)=0$, and $P_v^{-1}0=0$ since $P_v^{-1}$ is linear. Thus, the result follows since $\exp_{\theta^*}$ is a radial isometry.
\end{proof}

\subsection{The Horizontal and Vertical Spaces}
\label{sec:horizontalvertical}

In our setting, we can study the derivatives of $L^0$ at $\theta^*$ via the derivatives of the composition $\bar{L}^0=L^0\circ\pi$ at $\bar\theta^*$, since $\bar{L}^0$ is a map between Euclidean spaces, and similar for $L^*$. To do so, we need to characterize the tangent space $T_{\theta^*}\Theta$ of $\Theta$ at $\theta^*$. The theory of quotient manifolds provides a useful characterization of the structure of the tangent space $T_{\theta^*}\Theta$ in terms of the tangent space $T_{\bar\theta^*}\bar\Theta$ for any $\bar\theta^*\in\pi^{-1}(\theta^*)$. First, define the \emph{vertical space} $V_{\bar\theta^*}\subseteq T_{\bar\theta^*}\Theta$ to be the kernel of the derivative of the projection---i.e., the kernel of $D\pi(\bar\theta^*):T_{\bar\theta^*}\bar\Theta\to T_{\theta^*}\Theta$. Intuitively, the vertical space is the space of directions that are ``collapsed'' by the projection $\pi$ near $\bar\theta^*$. Then, define the \emph{horizontal space} to be its orthogonal complement---i.e., $H_{\bar\theta^*}=V_{\bar\theta^*}^\perp$. Finally, $T_{\theta^*}\Theta$ is naturally isomorphic $H_{\bar\theta^*}$; this isomorphism is called the \emph{horizontal lift}, and is denoted by $\lift_{\bar\theta^*}:T_{\theta^*}\Theta\to H_{\bar\theta^*}$. Importantly, $\lift_{\bar\theta^*}$ preserves the inner product (see Eq.~(9.31) in \cite{boumal2023introduction}).

We have already described the structure of the horizontal and vertical spaces, but we provide a sketch of the derivation here. In particular, we can characterize $V_{\bar\theta^*}$ in terms of the group action. Intuitively, the directions projected out by $\pi$ near $\bar\theta^*$ are exactly the ones corresponding to infinitesimal rotations of $\bar\theta^*$. Indeed, $V_{\bar\theta^*}$ equals the image of the derivative of action $F:\text{O}(k)\to\bar\Theta$ given by $F(U)=\bar\theta^*U$ at the identity $I\in\text{O}(k)$. Now, the derivative of $F$ at $I$ is the linear map $DF(I):\mathfrak{o}(k)\to T_{\bar\theta^*}\bar\Theta$ given by $DF(I)[Z]=\bar\theta^*Z$. Thus, the horizontal space is
\begin{align*}
H_{\bar\theta^*}=V_{\bar\theta^*}^\perp
\qquad\text{where}\qquad
V_{\bar\theta^*}=\ker D\pi(\bar\theta^*)=\im DF(I)=\{\bar\theta^*Z\mid Z\in\mathfrak{o}(k)\}.
\end{align*}

\subsection{Sketch of Proof of Theorem~\ref{thm:lowrank}}
\label{sec:thm:lowrank:proofsketch}

First, we have the following immediate consequence of our Taylor expansion established in Lemma~\ref{lem:taylor}.
\begin{corollary}
\label{cor:general}
Assume the setup of Lemma~\ref{lem:taylor}. Denote the minimum eigenvalue of $\hess L^0(\theta^*)$ by
\begin{align}
\label{eqn:lambda}
\lambda_{\text{min}}=\min_{w\in T_{\theta^*}\Theta}\frac{\|\hess L^0(\theta^*)[w]\|_{\theta^*}}{\|w\|_{\theta^*}},
\end{align}
and assume that (iii) $d(\theta^0,\theta^*)\le\lambda_{\text{min}}/K$. Then, we have
$d(\theta^0,\theta^*)\le2\|\grad L^0(\theta^*)\|_{\theta^*}/\lambda_{\text{min}}$.
\end{corollary}
\begin{proof}
By Lemma~\ref{lem:taylor}, triangle inequality, definition of $\lambda_{\text{min}}$, and since $\exp_{\theta^*}$ is a radial isometry, $\lambda_{\text{min}}\cdot d(\theta^0,\theta^*)-\|\grad L^0(\theta^*)\|_{\theta^*}\le K\cdot d(\theta^0,\theta^*)^2/2$. The claim follows by assumption (iii).
\end{proof}
Now, to prove Theorem~\ref{thm:lowrank}, we need to analyze the values $\grad L^0(\theta^*)$, $\lambda_{\text{min}}$, and $K$ in Corollary~\ref{cor:general}. Since $\grad L^0(\theta^*)$ is the sum of i.i.d. random variables, we can bound it using Hoeffding's inequality (see Appendix~\ref{sec:thm:lowrank:proof:1}). Our analysis of $K$ is quite involved but fairly mechanical; it follows upper bounding upper bounds on the derivatives of $\bar{L}^0$ (see Appendix~\ref{sec:thm:lowrank:proof:2}). Finally, since $\hess L^0(\theta^*)$ is also a sum of i.i.d. random variables, we can apply Hoeffding's inequality and show that $\lambda_{\text{min}}$ is close to the minimum eigenvalue $\lambda_{\text{min}}^*$ of $\hess L^*(\theta^*)$. Lower bounding $\lambda_{\text{min}}^*$ is equivalent to lower bounding the minimum eigenvalue of $\nabla_{\bar\theta}\bar{L}^*(\bar\theta^*)$ on $H_{\bar\theta^*}$, which corresponds to lower bounding $D_{\bar\theta}^2\bar{L}^*(\bar\theta^*)[Z,Z]$ over $Z\in H_{\bar\theta^*}$ with $\|Z\|_F=1$. By Assumption~\ref{assump:lowrank:lowerbound}, we have
\begin{align*}
D_{\bar\theta}^2\bar{L}^*(\bar\theta^*)[Z,Z]\ge\mu_0\lambda_0\|\bar\theta^*Z^\top+Z\bar\theta^{*\top}\|_F^2.
\end{align*}
A direct analysis based on the singular value decomposition of $\bar\theta^*$ shows $\|\bar\theta^*Z^\top+Z\bar\theta^{*\top}\|_F^2\ge2\sigma_{\text{min}}^2$, providing the desired lower bound (see Appendix~\ref{sec:thm:lowrank:proof:3}, specifically, the proof of Lemma~\ref{lem:lowrankhessian}). We give a complete proof of Theorem~\ref{thm:lowrank} in Appendix~\ref{sec:thm:lowrank:proof}.

\subsection{Proof of Theorem~\ref{thm:main}}
\label{sec:thm:main:proof}

Now, we leverage the Taylor expansion in Lemma~\ref{lem:taylor} to prove Theorem~\ref{thm:main}; given this result, our proof closely follows the traditional analysis outlined in Section~\ref{sec:thm:main:proof:overview}. By Lemma~\ref{lem:taylor}, the the linearity of $\lift_{\bar\theta^*}$ and the fact that $\lift_{\bar\theta^*}$ preserves inner products, we have
\begin{align}
\label{eqn:thm:asymptotic:1}
\|\lift_{\bar\theta^*}(\grad L^0(\theta^*))+\lift_{\bar\theta^*}(\hess L^0(\theta^*)[v])\|_{\theta^*}
\le\frac{K}{2}d(\theta^0,\theta^*)^2.
\end{align}
We can show that $\lift_{\bar\theta^*}(\grad L^0(\theta^*))=\grad\bar{L}^0(\bar\theta^*)$ (see Lemma~\ref{lem:quotientgradient0}) and $\lift_{\bar\theta^*}(\hess L^0(\theta^*)[v])
=\proj_{\bar\theta^*}^H(\hess\bar{L}^0(\bar\theta^*)[\bar{v}])$ (see Lemma~\ref{lem:quotienthessian}), where $\proj_{\bar\theta^*}^H$ denotes orthogonal projection onto $H_{\bar\theta^*}$. Furthermore, by Theorem 4.7 in \cite{massart2020quotient}, we have $v=\exp_{\bar\theta^*}^{-1}(\theta^0)=\bar\theta^0-\bar\theta^*$.

Let $\{e_i\}_{i=1}^{d'}$ be an orthonormal basis for $H_{\bar\theta^*}$ constructed as in Lemma~\ref{lem:invariant}, and let  $\phi^0$, $\phi^*$, $G^0$, $H^0$, and $H^*$ be the representations of $\bar\theta^0$, $\bar\theta^*$, $\nabla_{\bar\theta}\bar{L}^0(\bar\theta^*)$, $\nabla_{\bar\theta}^2\bar{L}^0(\bar\theta^*)$, and $\mathbb{E}[\nabla_{\bar\theta}^2\bar{L}^*(\bar\theta^*)]$, respectively; by definition, the representation of $v$ equals $\phi^0-\phi^*$. Also, let $g_i^0$, $h_i^0$, and $h_i^*$ be the representations of $\nabla_{\bar\theta}\bar{\ell}_i^0(\bar\theta^*)$, $\nabla_{\bar\theta}^2\bar{\ell}_i^0(\bar\theta^*)$, and $\nabla_{\bar\theta}^2\bar{\ell}_i^*(\bar\theta^*)$ in $\{e_i\}_{i=1}^n$, respectively, where $\bar{\ell}_i^0(\bar\theta)=\bar{\ell}(\langle X_i,\bar\theta\bar\theta^\top\rangle,y_i^*)$ and $\bar{\ell}_i^*(\bar\theta)=\mathbb{E}[\bar{\ell}_i^0]$; by definition, $H^0=n^{-1}\sum_{i=1}^nh_i^0$ and $H^*=n^{-1}\sum_{i=1}^nh_i^*$. By Lemma~\ref{lem:invariant}, $\phi^0$, $\phi^*$, $G^0$, $H^0$, and $H^*$ do not depend on the choice of $\bar\theta^*\in\pi^{-1}(\theta^*)$. Now, we can rewrite (\ref{eqn:thm:asymptotic:1}) as
\begin{align}
\label{eqn:thm:asymptotic:2}
\|G^0+H^0(\phi^0-\phi^*)\|_2
\le\frac{K}{2}d(\theta^0,\theta^*)^2.
\end{align}
By Theorem~\ref{thm:lowrank}, $\sqrt{n}\cdot d(\theta^0,\theta^*)^2\xrightarrow{P}0$ as $n\to\infty$, so by (\ref{eqn:thm:asymptotic:2}), $\sqrt{n}(G^0+H^0(\phi^0-\phi^*))\xrightarrow{P}0$, or equivalently, $\sqrt{n}(\phi^0-\phi^*)\xrightarrow{P}-(H^0)^{-1}\cdot\sqrt{n}G^0$. Furthermore, by the weak law of large numbers, we have $H^0\xrightarrow{P}H^*$, so by the continuous mapping theorem, we have $(H^0)^{-1}\xrightarrow{P}(H^*)^{-1}$, so by Slutsky's theorem, we have $\sqrt{n}(\phi^0-\phi^*)\xrightarrow{P}(H^*)^{-1}\cdot\sqrt{n}G^0$. By the central limit theorem, we have $G^0\xrightarrow{D}\mathcal{N}(0,\text{Cov}(G^0))$, where by Assumption~\ref{assump:lowrank:bartlett}, we have $\text{Cov}(G^0)=H^*$. Finally, by the delta method, we have ark$\sqrt{n}(\phi^0-\phi^*)\xrightarrow{D}\mathcal{N}(0,(H^*)^{-1})$, as claimed. \hfill $\blacksquare$

\section{Discussion}
\label{sec:discussion}

We have proven asymptotic normality for a generalized version of low-rank matrix sensing. As part of this goal, we have developed a number of novel techniques for proving such bounds based on Riemannian geometry. Rather than reason about the rotation that best aligns the estimated parameters with the true parameters, our approach is to reason about the eigenvalues of the Hessian on the subspace orthogonal to the Lie algebra of the rotation group; this approach is justified by the Riemannian geometry of quotient manifolds. We believe these techniques may be adapted to proving similar bounds for related problems in learning with low-rank matrices.

\paragraph{Limitations.}

Our work is limited in that we analyze the global minimizer and ignore optimization; however, prior work has extensively studied convergence of optimization for low-rank matrix sensing, and our statistical guarantees apply to the solutions that are found by these algorithms. Also, we focus on the symmetric case; a key challenge in the asymmetric case is that the quotient manifold is no longer Riemannian. Our analysis is also restricted to subgaussian noise, which is necessary since we use a slightly different proof of asymptotic normality which relies on the third derivative of the loss being Lipschitz continuous due to limitations of existing Taylor expansions on Riemannian manifolds. We leave addressing these limitations to future work.

%% file: appendix.tex
\section{Proof of Theorem~\ref{thm:lowrank}}
\label{sec:thm:lowrank:proof}

We prove Theorem~\ref{thm:lowrank} in five steps. First, we provide a high-probability bound on the gradient of $L^0$ at $\theta^*$ (Appendix~\ref{sec:thm:lowrank:proof:1}). Next, we provide a lower bound on the minimum eigenvalue of the Hessian of $L^0$ at $\theta^*$ (Appendix~\ref{sec:thm:lowrank:proof:2}). Third, we provide an upper bound on the Lipschitz constant of $L^0$ (Appendix~\ref{sec:thm:lowrank:proof:3}). Fourth, we provide a bound on the injectivity radius of the exponential map (Appendix~\ref{sec:thm:lowrank:proof:4}). Finally, we put these pieces together to prove the theorem (Appendix~\ref{sec:thm:lowrank:proof:final}). In addition, we provide a proof of Lemma~\ref{lem:invariant}, which is needed for Theorem~\ref{thm:main}, in Appendix~\ref{sec:lem:invariant:proof}.

Our results rely on lemmas expressing gradients and Hessians on a general quotient manifold in terms of corresponding quantities on the original manifold; these results are in Appendix~\ref{sec:quotient}.

\subsection{Gradient of the Empirical Loss}
\label{sec:thm:lowrank:proof:1}

\begin{lemma}
\label{lem:xbarbound}
Given $\delta\in\mathbb{R}_{>0}$, let $E_{\delta}$ be the event that
\begin{align}
\label{eqn:event}
\|\bar{X}\|_F\le\sqrt{\frac{8dk\sigma_{\epsilon}^2X_{\text{max}}^2\log(8dk/\delta)}{n}},
\end{align}
where $\bar{X}=n^{-1}\sum_{i=1}^n\epsilon_iX_i$ and $\epsilon_i=\bar\ell'(\langle X_i,\bar\theta^*\bar\theta^{*\top}\rangle,y_i^*)$. Then, under Assumption~\ref{assump:lowrank:upperbound}, we have $\mathbb{P}[E_{\delta}\mid\{X_i\}_{i=1}^n]\ge1-\delta/4$.
\end{lemma}
\begin{proof}
By Assumption~\ref{assump:lowrank:upperbound}, $\epsilon_iX_{ijh}$ is $\sigma_{\epsilon}X_{\text{max}}$ subgaussian for each $j\in[d]$ and $h\in[k]$; furthermore, $\mathbb{E}[\epsilon_iX_{ijh}\mid X_i]=\mu X_{ijh}=0$, where $\mu=\mathbb{E}[\epsilon_i\mid X_i]$ for all $i\in[n]$. By Hoeffding's inequality,
\begin{align*}
\mathbb{P}\left[\frac{1}{n}\left|\sum_{i=1}^n\epsilon_iX_{ijh}\right|\le\sqrt{\frac{8\sigma_{\epsilon}^2X_{\text{max}}^2\log(8dk/\delta)}{n}}\Biggm\vert\{X_i\}_{i=1}^n\right]\ge1-\frac{\delta}{4dk}
\end{align*}
for each $j\in[d]$ and $h\in[k]$. By a union bound, this inequality holds for all $j\in[d]$ and $h\in[k]$ with probability at least $1-\delta/4$. On the event that all of these inequalities hold, we have
\begin{align*}
\|\bar{X}\|_F=\sqrt{\sum_{j=1}^d\sum_{h=1}^k\bar{X}_{jh}^2}\le\sqrt{\frac{8dk\sigma_{\epsilon}^2X_{\text{max}}^2\log(8dk/\delta)}{n}},
\end{align*}
so the claim follows.
\end{proof}
\begin{lemma}
Under Assumption~\ref{assump:lowrank:upperbound}, given $\delta\in\mathbb{R}_{>0}$, on event $E_{\delta}$ as in (\ref{eqn:event}), we have
\begin{align*}
\|\grad L^0(\theta^*)\|_{\theta^*}\le\sqrt{\frac{32dk^2\sigma_{\text{max}}^2\sigma_{\epsilon}^2X_{\text{max}}^2\log(8dk/\delta)}{n}}.
\end{align*}
\end{lemma}
\begin{proof}
It is easy to check that the first derivative of the empirical loss is
\begin{align*}
D_{\bar\theta}\bar{L}^0(\bar\theta)[Z]
&=\frac{1}{n}\sum_{i=1}^n\bar\ell'(\langle X_i,\bar\theta\bar\theta^\top\rangle,y_i^*)\cdot\langle X_i,\bar\theta Z^\top+Z\bar\theta^\top\rangle,
\end{align*}
where $Z\in\mathbb{R}^{d\times r}$ is an element of the tangent space. Evaluating at $\bar\theta^*$, we have
\begin{align*}
D_{\bar\theta}\bar{L}^0(\bar\theta^*)[Z]
=\frac{1}{n}\sum_{i=1}^n\epsilon_i\cdot\langle X_i,\bar\theta^*Z^\top+Z\bar\theta^{*\top}\rangle
=\langle\bar{X},\bar\theta^*Z^\top+Z\bar\theta^{*\top}\rangle
=\langle(\bar{X}+\bar{X}^\top)\bar\theta^*,Z\rangle,
\end{align*}
where $\bar{X}=n^{-1}\sum_{i=1}^n\epsilon_iX_i$ and $\epsilon_i=\bar\ell'(\langle X_i,\bar\theta^*\bar\theta^{*\top}\rangle,y_i^*)$. Thus, $\grad\bar{L}^0(\bar\theta^*)=(\bar{X}+\bar{X}^\top)\bar\theta^*$, so
\begin{align*}
\|\grad\bar{L}^0(\bar\theta^*)\|_F
\le2\sigma_{\text{max}}\sqrt{k}\|\bar{X}\|_F.
\end{align*}
The claim follows from Lemma~\ref{lem:xbarbound}, and since $\grad L^0(\theta^*)=\grad\bar{L}^0(\bar\theta^*)$ by Lemma~\ref{lem:quotientgradient}.
\end{proof}

\subsection{Minimum Eigenvalue of the Hessian of the Empirical Loss}
\label{sec:thm:lowrank:proof:2}

\begin{lemma}
\label{lem:mbarbound}
Let $E_{\delta}'$ be the event that
\begin{align}
\label{eqn:mevent}
\max_{Z\in\mathbb{R}^{d\times k},\|Z\|_F=1}\|\bar{M}[Z]\|_F\le\sqrt{\frac{128dk^3X_{\text{max}}^4\sigma_{\text{max}}^4\sigma_{\epsilon}^2\log(8d^2k^2/\delta)}{n}},
\end{align}
where $\bar{M}[Z]=n^{-1}\sum_{i=1}^n(\epsilon_i'-\mu_i')\cdot\bar{M}_i[Z]$, where $\epsilon_i'=\bar\ell''(\langle X_i,\bar\theta^*\bar\theta^{*\top}\rangle,y_i^*)$ and $\mu'=\mathbb{E}[\epsilon_i'\mid X_i]$ for all $i\in[n]$, and where
\begin{align*}
\bar{M}_i[Z]=\langle (X_i+X_i^\top)\bar\theta^*,Z\rangle\cdot(X_i+X_i^\top)\bar\theta^*.
\end{align*}
Then, under Assumption~\ref{assump:lowrank:upperbound}, we have $\mathbb{P}[E_{\delta}'\mid\{X_i\}_{i=1}^n]\ge1-\delta/4$.
\end{lemma}
\begin{proof}
Let $\{E_{jh}\}_{j\in[d],h\in[k]}$ be a basis for $\mathbb{R}^{d\times k}$. Then, we have
\begin{align*}
|\langle\bar{M}_i[E_{jh}],E_{j'h'}\rangle|
&\le\|\text{vec}(X_i+X_i^\top)\|_{\infty}^2\cdot\|\text{vec}(\bar\theta^*)\|_1^2\cdot\|\text{vec}(Z)\|_1 \\
&\le4X_{\text{max}}^2\cdot k^2\|\bar\theta^*\|_1^2\cdot\sqrt{kd}\|Z\|_F \\
&\le4X_{\text{max}}^2\sigma_{\text{max}}^2k^{5/2}\sqrt{d},
\end{align*}
so by Assumption~\ref{assump:lowrank:upperbound}, $(\epsilon_i'-\mu')\cdot\langle\bar{M}_i[E_{jh}],E_{j'h'}\rangle$ is $4X_{\text{max}}^2\sigma_{\text{max}}^2\sigma_{\epsilon}k^{5/2}\sqrt{d}$ subgaussian. Thus, by Hoeffding's inequality, we have
\begin{align*}
\mathbb{P}\left[\left|\frac{1}{n}\sum_{i=1}^n(\epsilon_i'-\mu')\cdot\langle\bar{M}_i[E_{jh}],E_{j'h'}\rangle\right|\le\sqrt{\frac{128X_{\text{max}}^4\sigma_{\text{max}}^4\sigma_{\epsilon}^2k^5d\log(8d^2k^2/\delta)}{n}}\Biggm\vert\{X_i\}_{i=1}^n\right]\ge1-\frac{\delta}{4d^2k^2}.
\end{align*}
By a union bound, this inequality holds for all $j,j'\in[d]$ and $h,h'\in[k]$ with probability at least $1-\delta/4$. On the event that all of these inequalities hold, we have
\begin{align*}
\|\bar{M}[Z]\|_F
&=\frac{1}{n}\sum_{i=1}^n(\epsilon_i'-\mu')\cdot\left\langle\bar{M}_i\left[\sum_{j=1}^d\sum_{h=1}^kZ_{jh}E_{jh}\right],E_{j'h'}\right\rangle \\
&=\sum_{j=1}^d\sum_{h=1}^kZ_{jh}\left(\frac{1}{n}\sum_{i=1}^n(\epsilon_i'-\mu')\cdot\langle\bar{M}_i[E_{jh}],E_{j'h'}\rangle\right) \\
&\le\|Z\|_F\cdot\sqrt{\sum_{j=1}^d\sum_{h=1}^k\left(\frac{1}{n}\sum_{i=1}^n(\epsilon_i'-\mu')\cdot\langle\bar{M}      _i[E_{jh}],E_{j'h'}\rangle\right)^2} \\
&\le\sqrt{\frac{128d^2k^6X_{\text{max}}^4\sigma_{\text{max}}^4\sigma_{\epsilon}^2\log(8d^2k^2/\delta)}{n}}
\end{align*}
so the claim follows.
\end{proof}

\begin{lemma}
\label{lem:xbound}
Given $\delta\in\mathbb{R}_{>0}$, let $E_{\delta}''$ be the event that for all $M\in\mathbb{R}^{d\times d}$ satisfying $\text{rank}(M)\le2r$,
\begin{align}
\label{eqn:xevent}
\frac{1}{n}\sum_{i=1}^n\langle X_i,M\rangle^2\ge\frac{\lambda_0}{2}\|M\|_F^2.
\end{align}
Furthermore, suppose that
\begin{align*}
n\ge\frac{4d^2k^2X_{\text{max}}^4\log(4d^2k^2/\delta)}{\lambda_0}.
\end{align*}
Then, under Assumptions~\ref{assump:lowrank:upperbound} \&~\ref{assump:lowrank:lowerbound}, we have $\mathbb{P}[E_{\delta}'']\ge1-\delta/4$.
\end{lemma}
\begin{proof}
Let $\mathfrak{X}_i\in\mathbb{R}^{dk\times dk}$ be the matrix representation of the bilinear form $M,N\mapsto\langle X_i,M\rangle\cdot\langle X_i,N\rangle$, i.e., $\hat{\mathfrak{X}}_{i,jh,j'h'}=\langle X_i,E_{jh}\rangle\cdot\langle X_i,E_{j'h'}\rangle$, where $E_{jh}\in\mathbb{R}^{d\times d}$ is the matrix satisfying $E_{jh,j'h'}=1$ if $j'=j$ and $h'=h$ and $E_{jh,j'h'}=0$ otherwise; here, we have abused notation and used $jh$ to denote the index $(j-1)k+h\in[dk]$. Furthermore, let $\mathfrak{X}\in\mathbb{R}^{dk\times dk}$ be the matrix representation of the bilinear form $M,N\mapsto\mathbb{E}[\langle X,M\rangle\cdot\langle X,N\rangle]$, and let $\hat{\mathfrak{X}}=n^{-1}\sum_{i=1}^n\mathfrak{X}_i$ be the representation of $M,N\mapsto n^{-1}\sum_{i=1}^n\langle X_i,M\rangle\cdot\langle X_i,N\rangle$. Note that (\ref{eqn:xevent}) is equivalent to showing that the minimum eigenvalue of $\hat{\mathfrak{X}}$ is lower bounded by $\lambda_0/2$. To this end, note that by Assumption~\ref{assump:lowrank:upperbound}, we have $|\mathfrak{X}_{i,jh,j'h'}|\le X_{\text{max}}^2$, so by Hoeffding's inequality, for each $jh,j'h'\in[dk]$, we have
\begin{align*}
\mathbb{P}\left[|\hat{\mathfrak{X}}_{jh,j'h'}-\mathfrak{X}_{jh,j'h'}|\le\sqrt{\frac{2X_{\text{max}}^4\log(4d^2k^2/\delta)}{n}}\right]\ge1-\frac{\delta}{4d^2k^2}.
\end{align*}
By a union bound, this bound holds for all $jh,j'h'$ with probability at least $1-\delta/4$. On the event that all of these bounds hold, then letting $R_{jh}=\sum_{j'h'\neq jh}|(\hat{\mathfrak{X}}-\mathfrak{X})_{jh,j'h'}|$, we have
\begin{align*}
\lambda_{\text{min}}(\hat{\mathfrak{X}}-\mathfrak{X})
&\le\max_{jh}\left\{(\hat{\mathfrak{X}}-\mathfrak{X})_{jh}+R_{jh}\right\} \\
&\le\sum_{j'h'}|\hat{\mathfrak{X}}_{jh,j'h'}-\mathfrak{X}_{jh,j'h'}| \\
&\le\sqrt{\frac{2d^2k^2X_{\text{max}}^4\log(4d^2k^2/\delta)}{n}},
\end{align*}
where the first inequality follows by the Gershgorin circle theorem. Finally, the result follows from the fact that $\lambda_{\text{min}}(\hat{\mathfrak{X}})\ge\lambda_{\text{min}}(\mathfrak{X})-\lambda_{\text{min}}(\hat{\mathfrak{X}}-\mathfrak{X})$ and our assumption on $n$.
\end{proof}

\begin{lemma}
\label{lem:lowrankhessian}
Under Assumption~\ref{assump:lowrank:upperbound} \&~\ref{assump:lowrank:lowerbound}, given $\delta\in\mathbb{R}_{>0}$, letting $\lambda_{\text{min}}$ be as in (\ref{eqn:lambda}), on events $E_{\delta}$ as in (\ref{eqn:event}), $E_{\delta}'$ as in (\ref{eqn:mevent}), and $E_{\delta}''$ as in (\ref{eqn:xevent}), we have
\begin{align*}
\lambda_{\text{min}}\ge\mu_0\lambda_0\sigma_{\text{min}}^2
-\sqrt{\frac{160d^2k^6X_{\text{max}}^4\sigma_{\text{max}}^4\sigma_{\epsilon}^2\log(6d^2k^2/\delta)}{n}}.
\end{align*}
\end{lemma}
\begin{proof}
First, we compute the Hessian of the empirical loss. To this end, it is easy to check that the second derivatives of the empirical and true losses are the bilinear forms
\begin{align*}
D_{\bar\theta}^2\bar{L}^0(\bar\theta)[Z,W]
&=\frac{1}{n}\sum_{i=1}^n\bar\ell''(\langle X_i,\bar\theta\bar\theta^{\top}\rangle,y_i^*)\cdot\langle X_i,\bar\theta Z^\top+Z\bar\theta^\top\rangle\cdot\langle X_i,\bar\theta W^{\top}+W\bar\theta^\top\rangle \\
&\qquad\qquad+\bar\ell'(\langle X_i,\bar\theta\bar\theta^\top\rangle,y_i^*)\cdot\langle X_i,WZ^\top+ZW^{\top}\rangle \\
D_{\bar\theta}^2\bar{L}^*(\bar\theta)[Z,W]
&=\frac{1}{n}\sum_{i=1}^n\mathbb{E}[\bar\ell''(\langle X_i,\bar\theta\bar\theta^{\top}\rangle,y_i^*)\mid X_i]\cdot\langle X_i,\bar\theta Z^\top+Z\bar\theta^\top\rangle\cdot\langle X_i,\bar\theta W^{\top}+W\bar\theta^\top\rangle \\
&\qquad\qquad+\mathbb{E}[\bar\ell'(\langle X_i,\bar\theta\bar\theta^\top\rangle,y_i^*)\mid X_i]\cdot\langle X_i,WZ^\top+ZW^{\top}\rangle,
\end{align*}
respectively, where $Z,W\in\mathbb{R}^{d\times r}$ are elements of the tangent space. Evaluating the second derivative of the true loss at $\bar\theta^*$, we have
\begin{align*}
D_{\bar\theta}^2\bar{L}^*(\bar\theta^*)[Z,W]
&=\frac{1}{n}\sum_{i=1}^n\mu'\cdot\langle X_i,\bar\theta Z^\top+Z\bar\theta^\top\rangle\cdot\langle X_i,\bar\theta W^{\top}+W\bar\theta^\top\rangle \\
&=\left\langle\frac{1}{n}\sum_{i=1}^n\mu'\cdot\langle(X_i+X_i^\top)\bar\theta^*,Z\rangle\cdot(X_i+X_i^\top)\bar\theta^*,W\right\rangle,
\end{align*}
where $\mu'=\mathbb{E}[\epsilon_i'\mid X_i]$ for all $i\in[n]$ and $\epsilon_i'=\bar{\ell}''(\langle X_i,\bar\theta^*\bar\theta^{*\top}\rangle,y_i^*)$, and where we have used the fact that $\mu=0$ by Assumption~\ref{assump:lowrank:upperbound}, where $\mu=\mathbb{E}[\epsilon_i\mid X_i]$ for all $i\in[n]$ and $\epsilon_i=\bar\ell'(\langle X_i,\bar\theta^*\bar\theta^{*\top}\rangle,y_i^*)$. Thus, we have
\begin{align*}
\hess\bar{L}^*(\bar\theta^*)[Z]
=\frac{1}{n}\sum_{i=1}^n\mu'\cdot\langle(X_i+X_i^\top)\bar\theta^*,Z\rangle\cdot(X_i+X_i^\top)\bar\theta^*
\end{align*}
Now, evaluating the second derivative of the empirical loss at $\bar\theta^*$, we have
\begin{align*}
D_{\bar\theta}^2\bar{L}^0(\bar\theta^*)[Z,W]
&=\frac{1}{n}\sum_{i=1}^n\epsilon_i'\cdot\langle(X_i+X_i^\top)\bar\theta^*,Z\rangle\cdot\langle(X_i+X_i^\top)\bar\theta^*,W\rangle+\langle(\bar{X}+\bar{X}^\top)Z,W\rangle \\
&=\left\langle\frac{1}{n}\sum_{i=1}^n\epsilon_i'\cdot\langle(X_i+X_i^\top)\bar\theta^*,Z\rangle\cdot(X_i+X_i^\top)\bar\theta^*+(\bar{X}+\bar{X}^\top)Z,W\right\rangle \\
&=\left\langle\hess\bar{L}^*(\bar\theta^*)[Z]+M[Z]+(\bar{X}+\bar{X}^\top)Z,W\right\rangle,
\end{align*}
where $\bar{X}=n^{-1}\sum_{i=1}^n\epsilon_iX_i$ and
\begin{align*}
M[Z]&=\frac{1}{n}\sum_{i=1}^n(\epsilon_i'-\mu')\cdot\langle(X_i+X_i^\top)\bar\theta^*,Z\rangle\cdot(X_i+X_i^\top)\bar\theta^*.
\end{align*}
Thus, we have
\begin{align}
\label{eqn:lowrankhessian:5}
\hess\bar{L}^0(\bar\theta^*)[Z]=\hess\bar{L}^*(\bar\theta^*)[Z]+M[Z]+(\bar{X}+\bar{X}^\top)Z.
\end{align}
Now, to bound $\lambda_{\text{min}}$, according to Lemma~\ref{lem:quotienthessian}, we need to bound the projection of the Hessian onto the horizontal space. To this end, note that
\begin{align}
\|\proj^H_{\bar\theta^*}(\hess\bar{L}^0(\bar\theta^*)[Z])\|_F
&=\|\proj^H_{\bar\theta^*}(\hess\bar{L}^*(\bar\theta^*)[Z])+\proj^H_{\bar\theta^*}(M[Z])+\proj^H_{\bar\theta^*}((\bar{X}+\bar{X}^\top)Z)\|_F \nonumber \\
&\ge\|\proj^H_{\bar\theta^*}(\hess\bar{L}^*(\bar\theta^*)[Z])\|_F-\|\proj^H_{\bar\theta^*}(M[Z])\|_F-\|\proj^H_{\bar\theta^*}((\bar{X}+\bar{X}^\top)Z)\|_F \nonumber \\
&=\|\hess\bar{L}^*(\bar\theta^*)[Z]\|_F-\|\proj_{\bar\theta^*}^H(\bar{M}[Z])\|_F-\|\proj^H_{\bar\theta^*}((\bar{X}+\bar{X}^\top)Z)\|_F \nonumber \\
&\ge\|\hess\bar{L}^*(\bar\theta^*)[Z]\|_F-\|\bar{M}[Z]\|_F-\|(\bar{X}+\bar{X}^\top)Z\|_F \nonumber \\
&\ge\|\hess\bar{L}^*(\bar\theta^*)[Z]\|_F-\|\bar{M}[Z]\|_F-2\|\bar{X}\|_F\cdot\|Z\|_F,
\label{eqn:lowrankhessian:6}
\end{align}
where the first line follows by (\ref{eqn:lowrankhessian:5}) and linearity of the projection onto the horizontal space, the second by triangle inequality, the third by Lemma~\ref{lem:hessianproj} and the fact that $\bar\theta^*$ is a critical point of $\bar{L}^*$, and the fourth since the projection can only reduce the magnitude of the norm.

Because we are in linear space, the Hessian is symmetric positive semidefinite; furthermore, by Lemma~\ref{lem:hessianproj}, the vertical space $V_{\bar\theta^*}$ is in the kernel of $\hess\bar{L}^0(\bar\theta^*)$, so the $d-k$ smallest eigenvalues of $\hess\bar{L}^0(\bar\theta^*)$ are zero. Thus, $\|\hess\bar{L}^*(\bar\theta^*)[Z]\|_F$ is minimized by the eigenvector corresponding to the $(d-k+1)$th smallest eigenvalue of $\hess\bar{L}^0(\bar\theta^*)$. As a consequence, we have
\begin{align}
\min_{Z\in H_{\bar\theta^*},\|Z\|_F=1}\|\hess\bar{L}^*(\bar\theta^*)[Z]\|_F
&=\min_{Z\in H_{\bar\theta^*},\|Z\|_F=1}\langle\hess\bar{L}^*(\bar\theta^*)[Z],Z\rangle \nonumber \\
&=\min_{Z\in H_{\bar\theta^*},\|Z\|_F=1}D_{\bar\theta}^2\bar{L}^*(\bar\theta^*)[Z,Z].
\label{eqn:lowrankhessian:7}
\end{align}
Thus, by Lemma~\ref{lem:quotienthessian}, we have
\begin{align}
\lambda_{\text{min}}
&=\min_{Z\in H_{\bar\theta^*},\|Z\|_F=1}\|\proj(\hess\bar{L}^0(\bar\theta^*)[Z])\|_F \nonumber \\
&\ge\min_{Z\in H_{\bar\theta^*},\|Z\|_F=1}\left\{\|\hess\bar{L}^*(\bar\theta^*)[Z]\|_F-2\|\bar{X}\|_F-\|\bar{M}[Z]\|_F\right\} \nonumber \\
&\ge\min_{Z\in H_{\bar\theta^*},\|Z\|_F=1}\|\hess\bar{L}^*(\bar\theta^*)[Z]\|_F-2\|\bar{X}\|_F-\max_{Z\in\mathbb{R}^{d\times k},\|Z\|_F=1}\|\bar{M}[Z]\|_F \nonumber \\
&=\min_{Z\in H_{\bar\theta^*},\|Z\|_F=1}D_{\bar\theta}^2\bar{L}^*(\bar\theta^*)[Z,Z]-2\|\bar{X}\|_F-\max_{Z\in\mathbb{R}^{d\times k},\|Z\|_F=1}\|\bar{M}[Z]\|_F,
\label{eqn:lowrankhessian:8}
\end{align}
where the first line follows by Lemma~\ref{lem:quotienthessian}, the second by (\ref{eqn:lowrankhessian:6}), and the fourth by (\ref{eqn:lowrankhessian:7}). The second term of (\ref{eqn:lowrankhessian:8}) can be bounded by Lemma~\ref{lem:xbarbound} and the third by Lemma~\ref{lem:mbarbound}, so it remains to bound the first term. To this end, note that on event $E_{\delta}''$, we have
\begin{align}
D_{\bar\theta}^2\bar{L}^*(\bar\theta^*)[Z,Z]
&=\frac{1}{n}\sum_{i=1}^n\mu'\cdot\langle X_i,\bar\theta^*Z^\top+Z\bar\theta^{*\top}\rangle^2 \nonumber \\
&\ge\frac{\mu_0}{n}\sum_{i=1}^n\langle X_i,\bar\theta^*Z^\top+Z\bar\theta^{*\top}\rangle^2 \nonumber \\
&\ge\frac{\mu_0\lambda_0}{2}\cdot\|\bar\theta^*Z^\top+Z\bar\theta^{*\top}\|_F^2. \label{eqn:lowrankhessian:10}
\end{align}
To lower bound $\|\bar\theta^*Z^\top+Z\bar\theta^{*\top}\|_F^2$, let $\bar\theta^*=U\Sigma V^\top$ be the singular value decomposition (SVD) of $\bar\theta^*$, where $\Sigma=\begin{bmatrix}\Sigma_0 & 0\end{bmatrix}^\top$ and $\Sigma_0=\text{diag}(\sigma_1,...,\sigma_k)\in\mathbb{R}^{k\times k}$. Given $Z\in H_{\bar\theta^*}$, let $Y=U^\top ZV$ and $Y=\begin{bmatrix}Y_1 & Y_2\end{bmatrix}^\top$, where $Y_1\in\mathbb{R}^{k\times k}$ and $Y_2\in\mathbb{R}^{(d-k)\times k}$. Since $H_{\bar\theta^*}=V_{\bar\theta^*}^\top$, we have
\begin{align}
\label{eqn:lowrankhessian:3}
0
&=\langle Z,\bar\theta^*A\rangle
=\text{tr}(Z^\top\bar\theta^*A)
=\text{tr}(Z^\top U\Sigma V^\top A)
=\text{tr}(Z^\top\Sigma A')
=\langle Y,\Sigma A'\rangle
=\langle Y_1,\Sigma_0A'\rangle
\end{align}
for all $A\in\mathfrak{o}(k)$, and 
where $A'=V^\top AV\in\mathfrak{o}(k)$. Now, let $Y_1=\Sigma_0^{-1}(Y_1^S+Y_1^A)$, where $Y_1^S$ is symmetric and $Y_1^A$ is skew-symmetric; this representation exists since $\Sigma_0$ is invertible and since any matrix $W\in\mathbb{R}^{k\times k}$ can be decomposed as $W=W^S+W^A$, where $W^S=(W+W^\top)/2$ is symmetric and $W^A=(W-W^\top)/2$ is skew-symmetric. Continuing from (\ref{eqn:lowrankhessian:3}), we have
\begin{align}
\label{eqn:lowrankhessian:4}
0=\langle Y_1,\Sigma_0A'\rangle=\langle Y_1^S,A'\rangle+\langle Y_1^A,A'\rangle=\langle Y_1^A,A'\rangle.
\end{align}
Since $Y_1^A$ is skew-symmetric, and since (\ref{eqn:lowrankhessian:4}) holds for all $A'\in\mathfrak{o}(k)$, so $Y_1^A=0$. Then, continuing from (\ref{eqn:lowrankhessian:10}), we have
\begin{align}
D_{\bar\theta}^2\bar{L}^*(\bar\theta^*)[Z,Z]
&\ge\frac{\mu_0\lambda_0}{2}\cdot\|U(\Sigma Y^\top+Y\Sigma^\top)U^\top\|_F^2 \nonumber \\
&=\frac{\mu_0\lambda_0}{2}\cdot\|\Sigma Y^\top+Y\Sigma^\top\|_F^2 \nonumber \\
&=\frac{\mu_0\lambda_0}{2}\cdot\left(\|\Sigma_0Y_1^\top+Y_1\Sigma_0\|_F^2+\|\Sigma_0Y_2^\top\|_F^2+\|Y_2\Sigma_0\|_F^2\right) \nonumber \\
&=\frac{\mu_0\lambda_0}{2}\cdot\left(\|\Sigma_0Y_1^S\Sigma_0^{-1}+\Sigma_0^{-1}Y_1^S\Sigma_0\|_F^2+2\sigma_{\text{min}}^2\|Y_2\|_F^2\right).
\label{eqn:lowrankhessian:1}
\end{align}
Now, note that
\begin{align}
\|\Sigma_0Y_1^S\Sigma_0^{-1}+\Sigma_0^{-1}Y_1^S\Sigma_0\|_F^2
&=\|\Sigma_0Y_1^S\Sigma_0^{-1}\|_F^2+2\langle\Sigma_0Y_1^S\Sigma_0^{-1},\Sigma_0^{-1}Y_1^S\Sigma_0\rangle+\|\Sigma_0^{-1}Y_1^S\Sigma_0\|_F^2 \nonumber \\
&=2\|\Sigma_0Y_1^S\Sigma_0^{-1}\|_F^2
+\tr(\Sigma_0^{-1}Y_1^S\Sigma_0\Sigma_0^{-1}Y_1^S\Sigma_0) \nonumber \\
&=2\|\Sigma_0Y_1^S\Sigma_0^{-1}\|_F^2
+\tr(Y_1^SY_1^S) \nonumber \\
&=2\|\Sigma_0Y_1^\top\|_F^2
+\|Y_1^S\|_F^2 \nonumber \\
&\ge2\sigma_{\text{min}}^2\|Y_1\|_F^2.
\label{eqn:lowrankhessian:2}
\end{align}
Combining (\ref{eqn:lowrankhessian:1}) and (\ref{eqn:lowrankhessian:2}), for all $Z\in H_{\bar\theta^*}$, we have
\begin{align}
\label{eqn:lowrankhessian:9}
D_{\bar\theta}^2\bar{L}^*(\bar\theta^*)[Z,Z]\ge\mu_0\lambda_0\sigma_{\text{min}}^2\|Y\|_F^2=\mu_0\lambda_0\sigma_{\text{min}}^2\|Z\|_F^2=\mu_0\lambda_0\sigma_{\text{min}}^2.
\end{align}
Finally, the claim follows by substituting (\ref{eqn:lowrankhessian:9}), Lemma~\ref{lem:xbarbound}, and Lemma~\ref{lem:mbarbound} into (\ref{eqn:lowrankhessian:8}).
\end{proof}

\subsection{Lipschitz Constant of the Hessian of the Empirical Loss}
\label{sec:thm:lowrank:proof:3}

Let $V_{\bar\theta}=\{\bar\theta A\mid A\in\mathfrak{o}(k)\}$ be the vertical space, $H_{\bar\theta}=V_{\bar\theta}^\perp$ be the horizontal space, and $P^H(\bar\theta)[Z]$ and $P^V(\bar\theta)[Z]$ be the orthogonal projection of $Z$ onto $H_{\bar\theta}$ and $V_{\bar\theta}$, respectively.
\begin{lemma}
\label{lem:ebarbound}:
Given $\delta\in\mathbb{R}_{>0}$, let $E_{\delta}'''$ be the event that
\begin{align}
\label{eqn:eevent}
\bar\eta\le\mu_{\text{max}}+\left(3+\sqrt{\frac{72\log(12/\delta)}{n}}\right)\sigma_{\epsilon}
\qquad\forall\bar\eta\in\{\bar\epsilon,\bar\epsilon',\bar\epsilon''\},
\end{align}
where $\bar{\epsilon}=n^{-1}\sum_{i=1}^n|\epsilon_i|$, $\bar\epsilon'=n^{-1}\sum_{i=1}^n|\epsilon_i'|$, and $\bar\epsilon''=n^{-1}\sum_{i=1}^n|\epsilon_i''|$. Under Assumption~\ref{assump:lowrank:upperbound}, we have $\mathbb{P}[E_{\delta}'''\mid\{X_i\}_{i=1}^n]\ge1-\delta/4$.
\end{lemma}
\begin{proof}
Let $\bar\eta=n^{-1}\sum_{i=1}^n|\eta_i|$ and $\nu=\mathbb{E}[\eta_i\mid X_i]$ for all $i\in[n]$; by Assumption~\ref{assump:lowrank:upperbound}, $|\nu|\le\mu_{\text{max}}$ and $\eta_i-\nu$ is $\sigma_{\epsilon}$ subgaussian for all $i\in[n]$. By Proposition 2.5.2 in \citet{vershynin2018high}, we have $\mathbb{E}[|\eta_i-\nu|\mid X_i]\le3\sigma_{\epsilon}$, so $\mathbb{E}[|\eta_i|\mid X_i]\le|\nu|+3\sigma_{\epsilon}\le\mu_{\text{max}}+3\sigma_{\epsilon}$. Then, by Lemma 2.6.8 in \citet{vershynin2018high}, $|\eta_i|-\mathbb{E}[|\eta_i|\mid X_i]$ is $6\sigma_{\epsilon}$ subgaussian. Thus, by Hoeffding's inequality, we have
\begin{align*}
\mathbb{P}\left[\left(\frac{1}{n}\sum_{i=1}^n|\eta_i|\right)-\mu_{\text{max}}-3\sigma_{\epsilon}\le\sqrt{\frac{72\sigma_{\epsilon}^2\log(12/\delta)}{n}}\Biggm\vert\{X_i\}_{i=1}^n\right]\ge1-\frac{\delta}{12}.
\end{align*}
The claim follows by a union bound over the possible choices of $\bar\eta$.
\end{proof}
\begin{lemma}
\label{lem:lowranklipschitz:1}
Under Assumption~\ref{assump:lowrank:upperbound}, we have
\begin{align*}
\sup_{Z,W\in\mathbb{R}^{d\times k},\|Z\|_F=\|W\|_F=1}\|D_{\bar\theta}P^H(\bar\theta)[Z,W]\|_F
\le\frac{3}{\sigma_{\text{min}}}.
\end{align*}
\end{lemma}
\begin{proof}
Note that it suffices to prove the result for $P^V(\bar\theta)$, since $P^H(\bar\theta)=I-P^V(\bar\theta)$, where $I$ is the identity matrix, so $D_{\bar\theta}P^H(\bar\theta)[Z,W]=-D_{\bar\theta}P^V(\bar\theta)$. To show the result for $P^V(\bar\theta)$, suppose that $P^V(\bar\theta)[Z]=\bar\theta A$, and $P^V(\bar\theta')[Z]=\bar\theta'A'$ for some $A,A'\in\mathfrak{o}(k)$. Note that
\begin{align*}
\|Z-\bar\theta A\|_F
&\ge\|Z-\bar\theta'A\|_F-\|\bar\theta'-\bar\theta\|_F\cdot\|A\|_F \\
&\ge\|Z-\bar\theta'A'\|_F-\|\bar\theta'-\bar\theta\|_F\cdot\|A\|_F \\
&\ge\|Z-\bar\theta A'\|_F-\|\bar\theta'-\bar\theta\|_F\cdot(\|A\|_F+\|A'\|_F),
\end{align*}
where the second inequality follows by the property of the orthogonal projection that
\begin{align}
\label{eqn:lowranklipschitz:1}
P^V(\bar\theta')[Z]=\operatorname*{\arg\min}_{A''\in\mathfrak{o}(k)}\|Z-\bar\theta'A''\|_F,
\end{align}
i.e., the orthogonal projection onto the vertical space is the minimum distance element of the vertical space. In addition, note that
\begin{align}
\label{eqn:lowranklipschitz:2}
\|Z-\bar\theta A'\|_F
=\|(Z-\bar\theta A)+\bar\theta(A-A')\|_F
=\|Z-\bar\theta A\|_F+\|\bar\theta(A-A')\|_F,
\end{align}
where the second equality follows because $Z-\bar\theta A\in(V_{\bar\theta})^\perp$ and $\bar\theta(A-A')\in V_{\bar\theta}$ (since $A-A'\in\mathfrak{o}(k)$), so $\langle
Z-\bar\theta A,\bar\theta(A-A')\rangle=0$, i.e., they are orthogonal. Combining (\ref{eqn:lowranklipschitz:1}) and (\ref{eqn:lowranklipschitz:2}), we have
\begin{align}
\label{eqn:lowranklipschitz:3}
\|\bar\theta(A-A')\|_F\le\|\bar\theta'-\bar\theta\|_F\cdot(\|A\|_F+\|A'\|_F).
\end{align}
Next, note that
\begin{align}
\label{eqn:lowranklipschitz:4}
\|\bar\theta(A-A')\|_F\ge\|\bar\theta A-\bar\theta'A'\|_F-\|\bar\theta'-\bar\theta\|_F\cdot\|A'\|_F.
\end{align}
Combining (\ref{eqn:lowranklipschitz:3}) and (\ref{eqn:lowranklipschitz:4}), we have
\begin{align}
\label{eqn:lowranklipschitz:5}
\|P^V(\bar\theta)[Z]-P^V(\bar\theta')[Z]\|_F=\|\bar\theta A-\bar\theta'A'\|_F\le\|\bar\theta'-\bar\theta\|_F\cdot(\|A\|_F+2\|A'\|_F).
\end{align}
By Assumption~\ref{assump:lowrank:upperbound}, $\|Z\|_F\ge\|\bar\theta A\|_F\ge\sigma_{\text{min}}\|A\|_F$, and similarly for $A'$, so continuing from (\ref{eqn:lowranklipschitz:5}),
\begin{align*}
\|P^V(\bar\theta)[Z]-P^V(\bar\theta')[Z]\|_F\le\|\bar\theta'-\bar\theta\|_F\cdot\frac{3\|Z\|_F}{\sigma_{\text{min}}}.
\end{align*}
Thus, the function $F_Z:\mathbb{R}^{d\times k}\to\mathbb{R}^{d\times k}$ defined by $F_Z(\bar\theta)=P^V(\bar\theta)Z$ is Lipschitz continuous with Lipschitz constant $3\|Z\|_F/\sigma_{\text{min}}$. As a consequence, we have
\begin{align*}
\sup_{W\in\mathbb{R}^{d\times k},\|W\|_F=1}\|D_{\bar\theta}P^V(\bar\theta)[Z,W]\|_F
=\|D_{\bar\theta}F_Z(\bar\theta)[W]\|_{\text{op}}
\le\frac{3\|Z\|_F}{\sigma_{\text{min}}},
\end{align*}
where $DF_Z(\bar\theta):\mathbb{R}^{d\times k}\to\mathbb{R}^{d\times k}$, $\|\cdot\|_{\text{op}}$ is the operator norm on linear functions $\mathbb{R}^{d\times k}\to\mathbb{R}^{d\times k}$, and the equality follows by linearity of the derivative. The claim follows.
\end{proof}
\begin{lemma}
\label{lem:lowranklipschitz:2}
Under Assumption~\ref{assump:lowrank:upperbound}, on event $E_{\delta}'''$ as in (\ref{eqn:eevent}), we have
\begin{align*}
&\sup_{Z,V\in\mathbb{R}^{d\times k},\|Z\|_F=\|W\|_F=1}\|H^0(\bar\theta)[Z]\|_F \\
&\le
48X_{\text{max}}^3\sigma_{\text{max}}^4d^3k^2\left(K_{\ell}+\mu_{\text{max}}+\left(3+\sqrt{\frac{72\log(12/\delta)}{n}}\right)\sigma_{\epsilon}\right).
\end{align*}
\end{lemma}
\begin{proof}
Note that
\begin{align*}
\sup_{Z\in\mathbb{R}^{d\times k},\|Z\|_F=1}\|H^0(\bar\theta)[Z]\|_F
&=\sup_{Z,V\in\mathbb{R}^{d\times k},\|Z\|_F=\|V\|_F=1}|\langle H^0(\bar\theta)[Z],V\rangle| \\
&=\sup_{Z,V\in\mathbb{R}^{d\times k},\|Z\|_F=\|V\|_F=1}|D_{\bar\theta}^2\bar{L}^0(\bar\theta)[Z,W]|.
\end{align*}
Next, note that the second derivative of the empirical loss is
\begin{align*}
D_{\bar\theta}^2\bar{L}^0(\bar\theta)[Z,W]
&=\frac{1}{n}\sum_{i=1}^n\bar\ell''(\langle X_i,\bar\theta\bar\theta^{\top}\rangle,y_i^*)\cdot\langle X_i,\bar\theta Z^\top+Z\bar\theta^\top\rangle\cdot\langle X_i,\bar\theta W^{\top}+W\bar\theta^\top\rangle \\
&\qquad\qquad+\bar\ell'(\langle X_i,\bar\theta\bar\theta^\top\rangle,y_i^*)\cdot\langle X_i,WZ^\top+ZW^{\top}\rangle,
\end{align*}
By Assumption~\ref{assump:lowrank:upperbound}, for $\|Z\|_F=\|W\|_F=1$, we have $\|\bar\theta-\bar\theta^*\|_F\le2\sigma_{\text{max}}\sqrt{k}$, so we have
\begin{align*}
&|D_{\bar\theta}^2\bar{L}^0(\bar\theta)[Z,W]-D_{\bar\theta^2}\bar{L}^0(\bar\theta^*)[Z,W]| \\
&\le16K_{\ell}X_{\text{max}}^3\sigma_{\text{max}}^4d^3k^2+16K_{\ell}X_{\text{max}}^2\sigma_{\text{max}}^2d^2k+8K_{\ell}X_{\text{max}}^2\sigma_{\text{max}}^2d^2k \\
&\le48K_{\ell}X_{\text{max}}^3\sigma_{\text{max}}^4d^3k^2.
\end{align*}
Then, by Assumption~\ref{assump:lowrank:upperbound} and on event $E_{\delta}'''$, we have
\begin{align*}
|D_{\bar\theta}^2\bar{L}^0(\bar\theta^*)[Z,W]|
&\le4\bar{\epsilon}'X_{\text{max}}^2\sigma_{\text{max}}^2d^2k+2\bar{\epsilon}X_{\text{max}}d \\
&\le6X_{\text{max}}^2\sigma_{\text{max}}^2d^2k\left(\mu_{\text{max}}+\left(3+\sqrt{\frac{72\log(12/\delta)}{n}}\right)\sigma_{\epsilon}\right),
\end{align*}
where $\bar\epsilon=n^{-1}\sum_{i=1}^n|\epsilon_i|$ and $\epsilon_i=\bar\ell'(\langle X_i,\bar\theta^*\bar\theta^{*\top}\rangle,y_i^*)$, and $\bar\epsilon'=n^{-1}\sum_{i=1}^n|\epsilon_i'|$ and $\epsilon_i'=\bar\ell''(\langle X_i,\bar\theta^*\bar\theta^{*\top}\rangle,y_i^*)$, and where the second line follows by our assumption that $E_{\delta}'''$ holds. The claim follows.
\end{proof}
\begin{lemma}
\label{lem:lowranklipschitz:3}
Under Assumption~\ref{assump:lowrank:upperbound}, on event $E_{\delta}'''$ as in (\ref{eqn:eevent}), we have
\begin{align*}
&\sup_{Z,V\in\mathbb{R}^{d\times k},\|Z\|_F=\|V\|_F=1}\|D_{\bar\theta}H^0(\bar\theta)[Z,V]\|_F \\
&\le160X_{\text{max}}^4\sigma_{\text{max}}^5d^4k^{5/2}\left(K_{\ell}+\mu_{\text{max}}+\left(3+\sqrt{\frac{72\log(12/\delta)}{n}}\right)\sigma_{\epsilon}\right).
\end{align*}
\end{lemma}
\begin{proof}
Note that
\begin{align*}
\sup_{Z,V\in\mathbb{R}^{d\times k},\|Z\|_F=\|V\|_F=1}\|D_{\bar\theta}H^0(\bar\theta)[Z,V]\|_F
&=\sup_{Z,V,W\in\mathbb{R}^{d\times k},\|Z\|_F=\|V\|_F=\|W\|_F=1}|\langle D_{\bar\theta}H^0(\bar\theta)[Z,V],V\rangle| \\
&=\sup_{Z,V,W\in\mathbb{R}^{d\times k},\|Z\|_F=\|V\|_F=\|W\|_F=1}|D_{\bar\theta}\langle H^0(\bar\theta)[Z,V],W\rangle| \\
&=\sup_{Z,V,W\in\mathbb{R}^{d\times k},\|Z\|_F=\|V\|_F=\|W\|_F=1}|D_{\bar\theta}^3\bar{L}^0(\bar\theta)[Z,V,W]|.
\end{align*}
Next, note that the third derivative of the empirical loss is
\begin{align*}
D^3_{\bar\theta}\bar{L}(\bar\theta)[Z,W,V]
&=\frac{1}{n}\sum_{i=1}^n\bar\ell'''(\langle X_i,\bar\theta\bar\theta^\top\rangle,y_i^*)\cdot\langle X_i,\bar\theta Z^\top+Z\bar\theta^\top\rangle\cdot\langle X_i,\bar\theta W^\top+W\bar\theta^\top\rangle\cdot\langle X_i,\bar\theta V^\top+V\bar\theta^\top\rangle \\
&\qquad\qquad+\bar\ell''(\langle X_i,\bar\theta\bar\theta^\top\rangle,y_i^*)\cdot\langle X_i,WZ^\top+ZW^{\top}\rangle\cdot\langle X_i,\bar\theta V^\top+V\bar\theta^\top\rangle \\
&\qquad\qquad+\bar\ell''(\langle X_i,\bar\theta\bar\theta^\top\rangle,y_i^*)\cdot\langle X_i,VZ^\top+ZV^{\top}\rangle\cdot\langle X_i,\bar\theta W^\top+W\bar\theta^\top\rangle \\
&\qquad\qquad+\bar\ell''(\langle X_i,\bar\theta\bar\theta^\top\rangle,y_i^*)\cdot\langle X_i,VW^\top+WV^{\top}\rangle\cdot\langle X_i,\bar\theta Z^\top+Z\bar\theta^\top\rangle.
\end{align*}
By Assumption~\ref{assump:lowrank:upperbound}, for $\|Z\|_F=\|V\|_F=\|W\|_F=1$, we have $\|\bar\theta-\bar\theta^*\|_F\le2\sigma_{\text{max}}\sqrt{k}$, so we have
\begin{align*}
&|D_{\bar\theta}^3\bar{L}(\bar\theta)[Z,W,V]-D_{\bar\theta}^3\bar{L}(\bar\theta^*)[Z,W,V]| \\
&\le32K_{\ell}X_{\text{max}}^4\sigma_{\text{max}}^5d^4k^{5/2}
+48K_{\ell}X_{\text{max}}^3\sigma_{\text{max}}^3d^3k^{3/2}
+48K_{\ell}X_{\text{max}}^3\sigma_{\text{max}}^3\kappa d^3k^{3/2}
+24K_{\ell}X_{\text{max}}^2\sigma d^2\sqrt{k} \\
&\le160K_{\ell}X_{\text{max}}^4\sigma_{\text{max}}^5d^4k^{5/2}.
\end{align*}
Then, by Assumption~\ref{assump:lowrank:upperbound} and on event $E_{\delta}'''$, we have
\begin{align*}
|D_{\bar\theta}^3\bar{L}(\bar\theta^*)[Z,W,V]|
&\le8\bar\epsilon''X_{\text{max}}^3\sigma_{\text{max}}^3d^3k^{3/2}+12\bar\epsilon'X_{\text{max}}^2\sigma_{\text{max}}d^2\sqrt{k} \\
&\le20X_{\text{max}}^3\sigma_{\text{max}}^3d^3k^{3/2}\left(\mu_{\text{max}}+\left(3+\sqrt{\frac{72\log(12/\delta)}{n}}\right)\sigma_{\epsilon}\right),
\end{align*}
where $\bar\epsilon'=n^{-1}\sum_{i=1}^n|\epsilon_i'|$ and $\epsilon_i'=\bar\ell''(\langle X_i,\bar\theta^*\bar\theta^{*\top}\rangle,y_i^*)$, and $\bar\epsilon''=n^{-1}\sum_{i=1}^n|\epsilon_i''|$ and $\epsilon_i''=\bar\ell'''(\langle X_i,\bar\theta^*\bar\theta^{*\top}\rangle,y_i^*)$, and where the second line follows by our assumption that $E_{\delta}'''$ holds. The claim follows.
\end{proof}
\begin{lemma}
\label{lem:lowranklipschitz:4}
Under Assumption~\ref{assump:lowrank:upperbound}, assuming $n\ge\log(12/\delta)$ and on event $E_{\delta}'''$, $\hess\bar{L}^0$ is $K$-Lipschitz continuous, where
\begin{align}
\label{eqn:lowranklipschitz:main}
K&=160X_{\text{max}}^4\sigma_{\text{max}}^5d^4k^{5/2}(K_{\ell}+\mu_{\text{max}}+15\sigma_{\epsilon}).
\end{align}
\end{lemma}
\begin{proof}
The claim follows by using Lemmas~\ref{lem:lowranklipschitz:1}, \ref{lem:lowranklipschitz:2}, \&~\ref{lem:lowranklipschitz:3} to bound $K$ in Lemma~\ref{lem:euclideanlipschitz:2}, and since $3+\sqrt{72\log(12/\delta)/n}\le15$ under our assumption on $n$.
\end{proof}

\subsection{Injectivity Radius of the Exponential Map}
\label{sec:thm:lowrank:proof:4}

\begin{lemma}
\label{lem:lowrankinj}
Under Assumption~\ref{assump:lowrank:upperbound}, the injectivity radius of $\exp_{\theta^*}$ at $\theta^*$ is $\inj(\theta^*)\ge\sigma_{\text{min}}$.
\end{lemma}
\begin{proof}
We have $\inj(\theta^*)=\sigma_k(\bar\theta^*)\ge\sigma_{\text{min}}$, where the equality follows from Theorem 6.3 in \cite{massart2020quotient}, where $\sigma_k(M)$ is the $k$th largest singular value of $M$.
\end{proof}

\subsection{Proof of Theorem~\ref{thm:lowrank}}
\label{sec:thm:lowrank:proof:final}

\begin{proof}
Assume events $E_{\delta}$ as in (\ref{eqn:event}), $E_{\delta}'$ as in (\ref{eqn:mevent}), $E_{\delta}''$ as in (\ref{eqn:xevent}), and $E_{\delta}'''$ as in (\ref{eqn:eevent}). By our assumption on $n$ and by Lemma~\ref{lem:lowrankhessian}, $\lambda_{\text{min}}\ge\lambda_0\sigma_{\text{min}}^2/2$. Next, by our assumptions on $n$ and $d(\theta^0,\theta^*)$ and by Lemmas~\ref{lem:lowrankhessian} \&~\ref{lem:lowranklipschitz:4}, $d(\theta^0,\theta^*)\le\lambda_{\text{min}}/K$ (with $K$ as in (\ref{eqn:lowranklipschitz:main})). Finally, by our assumption on $d(\theta^0,\theta^*)$ and by Lemma~\ref{lem:lowrankinj}, $d(\theta^0,\theta^*)\le\inj(\theta^*)$. Thus, all of the assumptions of Corollary~\ref{cor:general}, so the claim follows by Corollary~\ref{cor:general} and using Lemmas~\ref{lem:xbarbound}, \ref{lem:mbarbound}, \ref{lem:xbound}, \&~\ref{lem:ebarbound} to bound $E_{\delta}$, $E_{\delta}'$, $E_{\delta}''$, and $E_{\delta}'''$.
\end{proof}

\subsection{Proof of Lemma~\ref{lem:invariant}}
\label{sec:lem:invariant:proof}

For the case $\phi^0$, we have
\begin{align*}
\phi^{0\prime}_i
=\langle\bar\theta^{0\prime},e_i'\rangle
=\langle\bar\theta^0U,e_iU\rangle
=\text{tr}(U^\top\bar\theta^{0\top}e_iU)
=\text{tr}(\bar\theta^{0\top}e_i)
=\langle\bar\theta^0,e_i\rangle
=\phi^0_i,
\end{align*}
as claimed. The case $\phi^*$ follows similarly. For $g^0$, we have
\begin{align*}
g_i^{0\prime}
=\langle\nabla_{\bar\theta}\bar\ell^0(\bar\theta^{*\prime}),e_i'\rangle
=D_{\bar\theta}\bar\ell^0(\bar\theta^{*\prime})[e_i]
&=\bar\ell'(\langle X,\bar\theta^{*\prime}\bar\theta^{*\prime\top}\rangle,y^*)\cdot\langle X,\bar\theta^{*\prime}e_i^\top+e_i\bar\theta^{*\prime\top}\rangle \\
&=\bar\ell'(\langle X,\bar\theta^*UU^\top\bar\theta^{*\top}\rangle,y^*)\cdot\langle X,\bar\theta^*UU^\top e_i^\top+e_iUU^\top\bar\theta^{*\top}\rangle \\
&=\bar\ell'(\langle X,\bar\theta^*\bar\theta^{*\top}\rangle,y^*)\cdot\langle X,\bar\theta^*e_i^\top+e_i\bar\theta^{*\top}\rangle \\
&=g^0_i.
\end{align*}
For the case $h^0$, we have
\begin{align*}
h^{0\prime}_{ij}
&=\langle e_i',\nabla_{\bar\theta}^2\bar\ell^0(\bar\theta^{*\prime})e_j'\rangle \\
&=D_{\bar\theta}^2\bar\ell^0(\bar\theta^{*\prime})[e_i',e_j'] \\
&=\bar\ell''(\langle X,\bar\theta^{*\prime}\bar\theta^{*\prime\top}\rangle,y^*)\cdot\langle X,\bar\theta^{*\prime}e_i^{\prime\top}+e_i'\bar\theta^{*\prime\top}\rangle\cdot\langle X,\bar\theta^{*\prime}e_j^{\prime\top}+e_j'\bar\theta^{*\prime\top}\rangle \\
&\qquad+\bar\ell'(\langle X,\bar\theta^{*\prime}\bar\theta^{*\prime\top}\rangle,y^*)\cdot\langle X,e_j'e_i^{\prime\top}+e_i'e_j^{\prime\top}\rangle \\
&=\bar\ell''(\langle X,\bar\theta^*UU^\top\bar\theta^{*\top}\rangle,y^*)\cdot\langle X,\bar\theta^*UU^\top e_i^\top+e_iUU^\top\bar\theta^{*\top}\rangle\cdot\langle X,\bar\theta^*UU^\top e_j^\top+e_jUU^\top\bar\theta^{*\top}\rangle \\
&\qquad+\bar\ell'(\langle X,\bar\theta^*UU^\top\bar\theta^{*\top}\rangle,y^*)\cdot\langle X,e_j UU^\top e_i^{\top}+e_iUU^\top e_j^{\top}\rangle \\
&=\bar\ell''(\langle X,\bar\theta^*\bar\theta^{*\top}\rangle,y^*)\cdot\langle X,\bar\theta^*e_i^\top+e_i\bar\theta^{*\top}\rangle\cdot\langle X,\bar\theta^*e_j^\top+e_j\bar\theta^{*\top}\rangle \\
&\qquad+\bar\ell'(\langle X,\bar\theta^*\bar\theta^{*\top}\rangle,y^*)\cdot\langle X,e_je_i^{\top}+e_ie_j^{\top}\rangle \\
&=h_{ij}^0.
\end{align*}
The case $h^*$ follows similarly. The claim follows. \hfill $\blacksquare$

\section{Learning on Riemannian Quotient Manifolds}
\label{sec:quotient}

We provide results that help bound the parameters in Corollary~\ref{cor:general} in the setting of Riemannian quotient manifolds; specifically, we consider the gradient of $L^0$ at $\theta^*$, the minimum eigenvalue of the Hessian of $L^0$ at $\theta^*$, and the Lipschitz constant of $L^0$. While the focus of our paper is on low-rank matrix sensing, the results in this section are fore general Riemannian quotient manifolds.

\subsection{Preliminaries}

We consider a parameter space $\bar\Theta$ with the same setup in Lemma~\ref{lem:taylor}. In addition, we consider a compact \emph{Lie group} $\mathcal{G}$, which is a group that is also a smooth compact manifold, and where multiplication map $(g,g')\mapsto gg'$ and inverse $g\mapsto g^{-1}$ are both smooth; we let $e\in\mathcal{G}$ denote the identity element of $\mathcal{G}$. Then, we consider a smooth action $\alpha:\mathcal{G}\times\bar\Theta\to\bar\Theta$ of $\mathcal{G}$ on $\bar\Theta$, which we also denote by $\alpha(g,\bar\theta)=g\bar\theta$; we assume this action is smooth, free, and isometric (it is also proper since $\mathcal{G}$ is assumed to be compact). Then, by Theorem 9.38 in \citet{boumal2023introduction}, the quotient space $\Theta=\bar\Theta/\mathcal{G}$ is a Riemannian manifold with the Riemannian metric given by
\begin{align*}
\langle v,w\rangle_{\theta}=\langle\bar v,\bar w\rangle_{\bar\theta},
\end{align*}
for all $v,w\in T_{\theta}\Theta$ and $\theta\in\Theta$, where $\bar\theta$ is such that $\theta=\pi(\bar\theta)$ (where $\pi:\bar\Theta\to\Theta$ is the quotient map), $\bar{v}=\lift_{\bar\theta}(v)$, and $\bar{w}=\lift_{\bar\theta}(w)$ (where  $\lift_{\bar\theta}:T_{\bar\theta}\bar\Theta\to T_{\theta}\Theta$ denotes the horizontal lift).

Next, we assume the true loss $\bar{L}^*:\bar\Theta\to\mathbb{R}$ and the empirical loss $\bar{L}^0:\bar\Theta\to\mathbb{R}$ are group invariant---i.e., $\bar{L}^*(g\bar\theta)=\bar{L}^*(\bar\theta)$ and $\bar{L}^0(g\bar\theta)=\bar{L}^0(\bar\theta)$; thus, we obtain smooth losses $L^*:\Theta\to\mathbb{R}_{\ge0}$ and $L^0:\Theta\to\mathbb{R}_{\ge0}$. We let $\bar\theta^*\in\operatorname*{\arg\min}_{\bar\theta\in\bar\Theta}\bar{L}^*(\bar\theta)$ and $\theta^*=\pi(\bar\theta^*)$, and $\bar\theta^0=\operatorname*{\arg\min}_{\bar\theta\in\bar\Theta}\bar{L}^0(\bar\theta)$ and $\theta^0=\pi(\bar\theta^0)$; it follows that $\theta^*\in\operatorname*{\arg\min}_{\theta\in\Theta}L^*(\theta)$ and $\theta^0\in\operatorname*{\arg\min}_{\theta\in\Theta}L^0(\theta)$.

\subsection{Gradient of the Empirical Loss}

Our first result expresses $\grad L^0(\theta^*)$ in terms of $\grad\bar{L}^0(\bar\theta^*)$.
\begin{lemma}
\label{lem:quotientgradient0}
We have $\lift_{\bar\theta^*}(\grad L^0(\theta^*))=\grad\bar{L}^0(\bar\theta^*)$.
\end{lemma}
\begin{proof}
This result follows from Proposition 9.39 in \cite{boumal2023introduction} and the fact that $\theta^*=\pi(\bar\theta^*)$.
\end{proof}
Then, we have the following straightforward consequence:
\begin{lemma}
\label{lem:quotientgradient}
We have $\|\grad L^0(\theta^*)\|_{\theta^*}=\|\grad\bar{L}^0(\bar\theta^*)\|_{\bar\theta^*}$.
\end{lemma}
\begin{proof}
This result follows from Lemma~\ref{lem:quotientgradient0} and Theorem 9.35 in \cite{boumal2023introduction}.
\end{proof}
In other words, to bound the norm of the gradient of $L^0$, it suffices to bound the gradient of $\bar{L}^0$.

\subsection{Minimum Eigenvalue of the Hessian of the Empirical Loss}

Next, we express the minimum eigenvalue $\lambda_{\text{min}}$ of $\hess L^0$ in terms of $\hess\bar{L}^0$.
\begin{lemma}
\label{lem:quotienthessian}
Letting $\lambda_{\text{min}}$ be as in (\ref{eqn:lambda}), we have $\lambda_{\text{min}}=\tilde\lambda_{\text{min}}$, where
\begin{align}
\label{eqn:altlambda}
\tilde\lambda_{\text{min}}=\min_{\bar{w}\in H_{\bar\theta^*}}\frac{\|\proj_{\bar\theta^*}^H(\hess\bar{L}^0(\bar\theta^*)[\bar{w}])\|_{\bar\theta^*}}{\|\bar{w}\|_{\bar\theta^*}}.
\end{align}
\end{lemma}
\begin{proof}
By the definition of the horizontal space (Definition 9.24 in \cite{boumal2023introduction}), we have that $H_{\bar\theta^*}$ is isomorphic to $T_{\theta^*}\Theta$ as inner product spaces via the horizontal lift $\bar{w}=\lift_{\bar\theta^*}(w)$. Thus, for any $w\in T_{\theta^*}\Theta$ and $\bar{w}\in H_{\bar\theta^*}$ such that $\bar{w}=\lift_{\bar\theta^*}(w)$, we have $\|w\|_{\theta^*}=\|\bar{w}\|_{\bar\theta^*}$. Also, we have
\begin{align}
\label{eqn:quotient:1}
\|\hess L^0(\theta^*)[w]\|_{\theta^*}^2
&=\langle\hess L^0(\theta^*)[w],\hess L^0(\theta^*)[w]\rangle_{\theta^*} \nonumber \\
&=\langle\text{lift}_{\bar\theta^*}(\hess L^0(\theta^*)[w]),\text{lift}_{\bar\theta^*}(\hess L^0(\theta^*)[w])\rangle_{\bar\theta^*} \nonumber \\
&=\langle\proj_{\bar\theta^*}^H(\hess\bar{L}^0(\bar\theta^*)[\text{lift}_{\bar\theta^*}(w)]),\proj_{\bar\theta^*}^H(\hess\bar{L}^0(\bar\theta^*)[\text{lift}_{\bar\theta^*}(w)])\rangle_{\bar\theta^*} \nonumber \\
&=\|\proj_{\bar\theta^*}^H(\hess\bar{L}^0(\bar\theta^*)[\text{lift}_{\bar\theta^*}(w)])\|_{\bar\theta^*}^2 \nonumber \\
&=\|\proj_{\bar\theta^*}^H(\hess\bar{L}^0(\bar\theta^*)[\bar{w}])\|_{\bar\theta^*}^2,
\end{align}
where the second equality follows by Eq.~(9.31) in~\cite{boumal2023introduction} and the third equality follows by Proposition~9.45 in~\cite{boumal2023introduction}. Combining (\ref{eqn:quotient:1}) with the fact that $\|w\|_{\theta^*}=\|\bar{w}\|_{\bar\theta^*}$, we have
\begin{align}
\label{eqn:quotient:2}
\frac{\|\hess L^0(\theta^*)[w]\|_{\theta^*}}{\|w\|_{\theta^*}}
=\frac{\|\proj_{\bar\theta^*}^H(\hess\bar{L}^0(\bar\theta^*)[\bar{w}])\|_{\bar\theta^*}}{\|\bar{w}\|_{\bar\theta^*}}.
\end{align}
Now, letting $\bar{w}\in H_{\bar\theta^*}$ achieve the minimum in (\ref{eqn:altlambda}), there is a unique vector $w\in T_{\theta^*}\Theta$ such that $\bar{w}=\lift_{\bar\theta^*}(w)$; by (\ref{eqn:quotient:2}), using $w$ in the objective of (\ref{eqn:lambda}) achieves the same value as (\ref{eqn:altlambda}), so $\lambda_{\text{min}}\le\tilde\lambda_{\text{min}}$. Conversely, letting $w\in T_{\theta^*}\Theta$ achieve the minimum in (\ref{eqn:lambda}), using $\bar{w}=\lift_{\bar\theta^*}(w)$ in the objective of (\ref{eqn:altlambda}) achieves the same value as (\ref{eqn:lambda}), so $\tilde\lambda_{\text{min}}\ge\lambda_{\text{min}}$. The claim follows.
\end{proof}
To compute (\ref{eqn:altlambda}), we need to minimize over the horizontal space $\bar{w}\in H_{\bar\theta^*}$; our next result characterizes this space as the image of the Lie algebra $\mathfrak{g}$ of $\mathcal{G}$ under the orbit of $\theta^*$ under $\mathcal{G}$.
\begin{lemma}
\label{lem:horizontal}
Letting $F:\mathcal{G}\to\Theta$ be defined by $F(g)=g\theta^*$, then $H_{\bar\theta^*}=(\im DF(e))^\perp$. Note that $DF(e):\mathfrak{g}\to T_{\bar\theta^*}\Theta$ is an isomorphism from the Lie algebra $\mathfrak{g}$ of $\mathcal{G}$ to its image $\im DF(e)\subseteq T_{\bar\theta^*}\bar\Theta$.
\end{lemma}
\begin{proof}
By Definition 9.24 in \cite{boumal2023introduction}, we have $H_{\bar\theta^*}=V_{\bar\theta^*}^\perp$ , so the claim follows from
\begin{align}
\label{eqn:horizontal:1}
V_{\bar\theta^*}
=\ker D\pi(\bar\theta^*)
=T_{\bar\theta^*}\pi^{-1}(\theta^*)
=\im DF(e).
\end{align}
The first equality follows by Definition 9.24 in \cite{boumal2023introduction} and the second by Proposition 9.3 in \cite{boumal2023introduction} (and since $\pi(\bar\theta^*)=\theta^*$). For the third equality, by the definitions of $F$ and $\pi$, $\im F=\pi^{-1}(\theta^*)$. Then, since the action of $\mathcal{G}$ is smooth, free, and proper, by Proposition 21.7 in \cite{lee2003introduction}, $F$ is a smooth embedding, so by Proposition 5.2 in \cite{lee2003introduction}, the map $F:\mathcal{G}\to\pi^{-1}(\theta^*)$ is a diffeomorphism. Then, by Proposition 3.6 in \cite{lee2003introduction}, $DF(e):\mathfrak{g}\to T_{\bar\theta^*}\pi^{-1}(\theta^*)$ is an isomorphism, where $\mathfrak{g}=T_e\mathcal{G}$. Thus, the third equality in (\ref{eqn:horizontal:1}) follows, so the claim follows.
\end{proof}
Our next result helps characterize the projection of the Hessian onto the horizontal space.
\begin{lemma}
\label{lem:hessianproj}
For any $\bar{L}:\bar\Theta\to\mathbb{R}$ that is invariant under $\mathcal{G}$, if $\bar\theta\in\bar\Theta$ is a critical point of $\bar{L}$, then
\begin{align*}
\proj_{\bar\theta}^H(\hess\bar{L}(\bar\theta)[\bar{w}])=\hess\bar{L}(\bar\theta)[\bar{w}].
\end{align*}
\end{lemma}
\begin{proof}
By Lemma 9.41 in \cite{boumal2023introduction}, the vertical space $V_{\bar\theta}$ of $T_{\bar\theta}\bar\Theta$ at $\bar\theta$ is in the kernel of $\hess\bar{L}(\bar\theta)$, i.e., $\hess\bar{L}(\bar\theta)[\bar{v}]=0$ for all $\bar{v}\in V_{\bar\theta}$. Thus, for any $\bar{w}\in T_{\bar\theta}\bar\Theta$ and $\bar{v}\in V_{\bar\theta}$, we have
\begin{align*}
\langle\bar{v},\hess\bar{L}(\bar\theta)[\bar{w}]\rangle_{\bar\theta}
=\langle\hess\bar{L}(\bar\theta)[\bar{v}],\bar{w}\rangle_{\bar\theta}=0,
\end{align*}
where the first equality follows by Proposition 5.15 in \cite{boumal2023introduction}. Thus, $\hess\bar{L}(\bar\theta)[\bar{w}]$ is orthogonal to $V_{\bar\theta}$, so it is in the horizontal space. The claim follows.
\end{proof}

\subsection{Lipschitz Constant of the Hessian of the Empirical Loss}

Finally, we consider the Lischitz constant $K$ of $\hess L^0$. To this end, let $\nabla:T\Theta\times\mathfrak{X}(\Theta)\to T\Theta$ denote the Riemannian connection on $\Theta$, where $T\Theta$ is the tangent bundle of $\Theta$, and
similarly let $\bar{\nabla}:\mathfrak{X}(\bar\Theta)\to T\bar\Theta$ denote the Riemannian connection on $\bar\Theta$. By Corollary 10.52 in \cite{boumal2023introduction}, $\hess L^0$ is $K$-Lipschitz continuous iff
\begin{align}
\label{eqn:quotientlipschitz:1}
\sup_{v\in T_{\theta}\Theta,\|v\|_{\theta}=1}\|H_{\theta,w}(v)\|_{\theta}\le K\|w\|_{\theta}
\qquad(\forall\theta\in\Theta,w\in T_{\theta}\Theta),
\end{align}
where we define $H_{\theta,w}$ as follows. First, we define $H:\mathfrak{X}(\Theta)\times\mathfrak{X}(\Theta)\to\mathfrak{X}(\Theta)$ by
\begin{align}
\label{eqn:quotientlipschitz:2}
H(V,W)=\nabla_W(\hess L^0(V))-\hess L^0(\nabla_WV).
\end{align}
Note that $H(V,W)$ is Eq.~10.49 in \cite{boumal2023introduction}. Here, we use the alternative form of the Hessian as an operator $\hess L^0:\mathfrak{X}(\Theta)\to\mathfrak{X}(\Theta)$. Now, given $\theta\in\Theta$ and $w\in T_{\theta}\Theta$, we define the linear map $H_{\theta,w}:T_{\theta}\Theta\to T_{\theta}\Theta$ by $H_{\theta,w}(v)=H(V,W)(\theta)$, where $V,W\in\mathfrak{X}(\Theta)$ are any vector fields satisfying $V(\theta)=v$, and $W(\theta)=w$; this definition does not depend on the specific choice of $V$ and $W$, so it is well-defined (see Definition 10.77 in \cite{boumal2023introduction}).

To apply (\ref{eqn:quotientlipschitz:1}) to obtain the Lipschitz constant $K$ of $\hess L^0$, we need to bound $\|H_{\theta,w}(v)\|_{\theta}$; our next result helps by providing an expression of $H_{\theta,w}$ in terms of values on $\bar\Theta$.
\begin{lemma}
\label{lem:quotientlipschitz}
We have
\begin{align}
\label{eqn:quotientlipschitz:6}
\lift(H_{\theta,w}(v))=\proj^H_{\bar\theta}(\bar{\nabla}_{\bar{w}}(\proj^H(\hess\bar{L}^0(\bar{V}))))-\proj^H_{\bar\theta}(\hess\bar{L}^0(\bar\theta)[\proj_{\bar\theta}^H(\bar{\nabla}_{\bar{w}}\bar{V})]),
\end{align}
where the horizontal lift for vector fields $\lift:\mathfrak{X}(\Theta)\to\mathfrak{X}(\bar\Theta)$ is the horizontal lift applied pointwise, the horizontal projection for vector fields $\proj^H:\mathfrak{X}(\bar\Theta)\to\mathfrak{X}(\bar\Theta)$ is the horizontal projection applied pointwise, and $\bar{V}=\lift(V)$; in addition, the first occurrence of $\hess\bar{L}^0$ is as a mapping $\hess\bar{L}^0:\mathfrak{X}(\bar\Theta)\to\mathfrak{X}(\bar\Theta)$, whereas the second is as a linear map $\hess\bar{L}^0(\bar\theta):T_{\bar\theta}\bar\Theta\to T_{\bar\theta}\bar\Theta$.
\end{lemma}
\begin{proof}
The horizontal lift of the first term in (\ref{eqn:quotientlipschitz:2}) is
\begin{align}
\lift(\nabla_W(\hess L^0(V)))
&=\proj^H(\bar{\nabla}_{\bar W}(\lift(\hess L^0(V)))) \nonumber \\
&=\proj^H(\bar{\nabla}_{\bar W}(\proj^H(\hess\bar{L}^0(\bar{V}))),
\label{eqn:quotientlipschitz:3}
\end{align}
where $\bar{W}=\lift(W)$, and where the first equality in (\ref{eqn:quotientlipschitz:3}) follows by Theorem 9.43 in \cite{boumal2023introduction} and the second by Proposition 9.45. The horizontal lift of the second term in (\ref{eqn:quotientlipschitz:2}) is
\begin{align}
\lift(\hess L^0(\nabla_WV))
&=\proj^H(\hess\bar{L}^0(\lift(\nabla_WV))) \nonumber \\
&=\proj^H(\hess\bar{L}^0(\proj^H(\bar{\nabla}_{\bar{W}}\bar{V}))),
\label{eqn:quotientlipschitz:4}
\end{align}
where the first equality follows by Proposition 9.45 in \cite{boumal2023introduction} and the second by Theorem 9.43. Next, combining combining (\ref{eqn:quotientlipschitz:2}), (\ref{eqn:quotientlipschitz:3}), and (\ref{eqn:quotientlipschitz:4}), and using the fact that the horizontal lift is linear (see Definition 9.25 in \cite{boumal2023introduction}), we have
\begin{align}
\label{eqn:quotientlipschitz:5}
\lift(H(V,W))
&=\proj^H(\bar{\nabla}_{\bar W}(\proj^H(\hess\bar{L}^0(\bar{V}))))-\proj^H(\hess\bar{L}^0(\proj^H(\bar{\nabla}_{\bar{W}}\bar{V}))).
\end{align}
Evaluating (\ref{eqn:quotientlipschitz:5}) at $\bar\theta$ gives
\begin{align}
\lift(H(V,W))(\bar\theta)=\proj^H(\bar{\nabla}_{\bar W}(\proj^H(\hess\bar{L}^0(\bar{V}))))(\bar\theta)-\proj^H(\hess\bar{L}^0(\proj^H(\bar{\nabla}_{\bar{W}}\bar{V})))(\bar\theta). \label{eqn:quotientlipschitz:7}
\end{align}
For the first term of (\ref{eqn:quotientlipschitz:7}), letting $\bar{w}=\bar{W}(\bar\theta)=\lift_{\bar\theta}(w)$, we have
\begin{align}
\proj^H(\bar{\nabla}_{\bar W}(\proj^H(\hess\bar{L}^0(\bar{V}))))(\bar\theta)
&=\proj^H_{\bar\theta}(\bar{\nabla}_{\bar W}(\proj^H(\hess\bar{L}^0(\bar{V})))(\bar\theta)) \nonumber \\
&=\proj^H_{\bar\theta}(\bar{\nabla}_{\bar w}(\proj^H(\hess\bar{L}^0(\bar{V})))),
\label{eqn:quotientlipschitz:8}
\end{align}
where the first line follows by the definition of $\proj^H$, and the second by Proposition 5.21 in \cite{boumal2023introduction}, which shows the Riemannian connection $\nabla$ only depends pointwise on $\bar{W}$. We cannot further simplify (\ref{eqn:quotientlipschitz:8}) since $\nabla_{\bar w}$ takes a vector field as input; in particular, we cannot convert the dependence on $\bar{V}$ into a dependence on $\bar{v}$ alone. Intuitively, the second term in (\ref{eqn:quotientlipschitz:7}) ``subtracts off'' the non-pointwise dependence on $\bar{V}$ of the first term. For the second term of (\ref{eqn:quotientlipschitz:7}), we have
\begin{align}
\label{eqn:quotientlipschitz:9}
\proj^H(\hess\bar{L}^0(\proj^H(\bar{\nabla}_{\bar{W}}\bar{V})))(\bar\theta)
&=\proj^H_{\bar\theta}(\hess\bar{L}^0(\proj^H(\bar{\nabla}_{\bar{W}}\bar{V}))(\bar\theta)) \\
&=\proj^H_{\bar\theta}(\hess\bar{L}^0(\bar\theta)[\proj^H(\bar{\nabla}_{\bar{W}}\bar{V})(\bar\theta)]) \nonumber \\
&=\proj^H_{\bar\theta}(\hess\bar{L}^0(\bar\theta)[\proj_{\bar\theta}^H(\bar{\nabla}_{\bar{W}}\bar{V}(\bar\theta))]) \nonumber \\
&=\proj^H_{\bar\theta}(\hess\bar{L}^0(\bar\theta)[\proj_{\bar\theta}^H(\bar{\nabla}_{\bar{w}}\bar{V})]),
\end{align}
where the first line follows by definition of $\proj^H$, the second since the Hessian only depends pointwise on $\bar\theta$ (see Definition 5.14 in \cite{boumal2023introduction}), the third since $\proj^H$ is defined pointwise, and the fourth by Proposition 5.21 in \cite{boumal2023introduction} as before. As with (\ref{eqn:quotientlipschitz:8}), we cannot simplify any further. Finally, the claim follows by plugging (\ref{eqn:quotientlipschitz:8}) and (\ref{eqn:quotientlipschitz:9}) into (\ref{eqn:quotientlipschitz:7}).
\end{proof}
While we cannot further simplify $H_{\theta,w}(v)$ in general, we can do so when $\bar\Theta=\mathbb{R}^d$. In this case, we can choose a global (orthonormal) frame, namely, the standard basis. Using this global frame, the horizontal projection has the form
\begin{align}
\label{eqn:euclideanlipschitz:1}
\proj^H(\bar{V})(\bar\theta)=P^H(\bar\theta)\bar{V}(\bar\theta),
\end{align}
where we abuse notation and use $\bar{V}$ to denote both a vector field (left-hand side) and the vector-valued function $\bar{V}:\bar\Theta\to\mathbb{R}^d$ (right-hand side). That is, $\proj^H$ is given by multiplying $\bar{V}(\bar\theta)$ by $P^H(\bar\theta)$, where $P^H$ is some matrix-valued function $P^H:\bar\Theta\to\mathbb{R}^{d\times d}$. Similarly, $\hess\bar{L}^0$ has form
\begin{align}
\label{eqn:euclideanlipschitz:2}
\hess\bar{L}^0(\bar{V})(\bar\theta)=H^0(\bar\theta)\bar{V}(\bar\theta),
\end{align}
for some matrix-valued function $H^0:\bar\Theta\to\mathbb{R}^{d\times d}$. In addition, by Theorem 5.7 in \cite{boumal2023introduction}, for the Riemannian connection $\bar{\nabla}$, we have
\begin{align}
\label{eqn:euclideanlipschitz:3}
\bar{\nabla}_{\bar{w}}\bar{V}(\bar\theta)
=D_{\bar\theta}\bar{V}(\bar\theta)[\bar{w}],
\end{align}
where $D_{\bar\theta}\bar{V}:\bar\Theta\to\mathbb{R}^{d\times d}$ is the Jacobian, and we have abused notation and used $\bar{w}$ to denote both an element of the tangent bundle (left-hand side) and the corresponding vector $\bar{w}\in\mathbb{R}^d$ induced by our choice of global frame. For derivatives of vector-valued functions, $D_{\bar\theta}\bar{V}(\bar\theta)[\bar{w}]=D_{\bar\theta}\bar{V}(\bar\theta)\bar{w}$ is given by matrix multiplication, but we retain this notation since we will use it to denote derivatives of matrix-valued functions below. Now, we have the following:
\begin{lemma}
\label{lem:euclideanlipschitz:1}
If $\bar\Theta=\mathbb{R}^d$, then $\lift(H_{w,\theta}(v))=P^H(\bar\theta)H_{\bar{w}}(\bar\theta)\bar{v}$, where by abuse of notation we let $\bar{v}\in\mathbb{R}^d$ denote the vector induced by our choice of global frame, and where $H_{\bar{w}}:\bar\Theta\to\mathbb{R}^{d\times d}$ is
\begin{align*}
H_{\bar{w}}(\bar\theta)=(D_{\bar\theta}P^H(\bar\theta)[\bar{w}])H^0(\bar\theta)
+D_{\bar\theta}H^0(\bar\theta)[\bar{w}]
+H^0(\bar\theta)(D_{\bar\theta}P^H(\bar\theta)[\bar{w}])
\end{align*}
\end{lemma}
\begin{proof}
First, consider the first term of (\ref{eqn:quotientlipschitz:6}); we have
\begin{align}
&\proj^H_{\bar\theta}(\bar{\nabla}_{\bar{w}}(\proj^H(\hess\bar{L}^0(\bar{V})))) \nonumber \\
&=P^H(\bar\theta)D_{\bar\theta}(P^H(\bar\theta)H^0(\bar\theta)\bar{V}(\bar\theta))[\bar{w}] \nonumber \\
&=P^H(\bar\theta)(D_{\bar\theta}P^H(\bar\theta)[\bar{w}])H^0(\bar\theta)\bar{V}(\bar\theta)
+P^H(\bar\theta)(D_{\bar\theta}H^0(\bar\theta)[\bar{w}])\bar{V}(\bar\theta)
+P^H(\bar\theta)H^0(\bar\theta)(D_{\bar\theta}\bar{V}(\bar\theta)[\bar{w}])
\label{eqn:euclideanlipschitz:4}
\end{align}
where the first equality follows by (\ref{eqn:euclideanlipschitz:1}), (\ref{eqn:euclideanlipschitz:2}), and (\ref{eqn:euclideanlipschitz:3}), and the second by the product rule and the fact that $P^H(\bar\theta)^2=P^H(\bar\theta)$ since $P^H(\bar\theta)$ is a projection. Next, for the second term of (\ref{eqn:quotientlipschitz:6}), we have
\begin{align}
\label{eqn:euclideanlipschitz:5}
&\proj^H_{\bar\theta}(\hess\bar{L}^0(\bar\theta)[\proj_{\bar\theta}^H(\bar\nabla_{\bar{w}}\bar{V})]) \nonumber \\
&=P^H(\bar\theta)H^0(\bar\theta)P^H(\bar\theta)(D_{\bar\theta}\bar{V}(\bar\theta)[\bar{w}]) \nonumber \\
&=P^H(\bar\theta)H^0(\bar\theta)(D_{\bar\theta}\bar{V}(\bar\theta)[\bar{w}])
-P^H(\bar\theta)H^0(\bar\theta)(D_{\bar\theta}P^H(\bar\theta)[\bar{w}])\bar{V}(\bar\theta),
\end{align}
where the first line follows by (\ref{eqn:euclideanlipschitz:1}), (\ref{eqn:euclideanlipschitz:2}), and (\ref{eqn:euclideanlipschitz:3}), and the second since we have assumed that $\bar{V}(\bar\theta)$ is in the horizontal space, so $P^H(\bar\theta)\bar{V}(\bar\theta)=\bar{V}(\bar\theta)$, which implies that
\begin{align*}
(D_{\bar\theta}P^H(\bar\theta)[\bar{w}])\bar{V}(\bar\theta)
+P^H(\bar\theta)(D_{\bar\theta}\bar{V}(\bar\theta)[\bar{w}])
=D_{\bar\theta}\bar{V}(\bar\theta)[\bar{w}].
\end{align*}
The claim follows by substituting (\ref{eqn:euclideanlipschitz:4}) and (\ref{eqn:euclideanlipschitz:5}) into (\ref{eqn:quotientlipschitz:6}) in Lemma~\ref{lem:quotientlipschitz} and since $\bar{V}(\bar\theta)=\bar{v}$.
\end{proof}
As a consequence, we have the following:
\begin{lemma}
\label{lem:euclideanlipschitz:2}
We have $\hess L^0$ is $K$-Lipschitz, where $K=\sup_{\bar{w},\bar{v}\in\mathbb{R}^d,\|\bar{w}\|_2=\|\bar{v}\|_2=1}\|H_{\bar{w}}(\bar\theta)\bar{v}\|_2$.
\end{lemma}
\begin{proof}
First, note that for any $v$ such that $\|v\|_{\theta}=1$, we have
\begin{align*}
\|\bar{v}\|_2=\|\bar{v}\|_{\bar\theta}=\|v\|_{\theta}=1,
\end{align*}
where the first equality follows since our global frame is orthonormal and the second by Theorem 9.35 in \cite{boumal2023introduction}. Then, we have
\begin{align}
\label{eqn:euclideanlipschitz:7}
\|H_{w,\theta}(v)\|_{\theta}
=\|\lift(H_{w,\theta}(v))\|_{\bar\theta}
=\|P^H(\bar\theta)H_{\bar{w}}(\bar\theta)\bar{v}\|_2
\le\|H_{\bar{w}}(\bar\theta)\bar{v}\|_2,
\end{align}
where the first equality follows by Theorem 9.35 in \cite{boumal2023introduction}, the second equality by Lemma~\ref{lem:euclideanlipschitz:1}, and the inequality since $P^H(\bar\theta)$ is an orthogonal projection, so it can only decrease the norm. The claim follows by substituting (\ref{eqn:euclideanlipschitz:7}) into (\ref{eqn:quotientlipschitz:1}).
\end{proof}